\documentclass{article}

\usepackage[final,nonatbib]{neurips_2022}

\usepackage[utf8]{inputenc} 
\usepackage[T1]{fontenc}    
\usepackage{hyperref}       
\usepackage{url}            
\usepackage{booktabs}       
\usepackage{amsfonts}       
\usepackage{nicefrac}       
\usepackage{microtype}      
\usepackage{xcolor}         
\usepackage{multirow}
\usepackage{xspace}
\usepackage{graphicx}
\usepackage{caption}
\usepackage{subcaption}
\usepackage{wrapfig}

\usepackage{aisecure_math}
\usepackage{amssymb}

\usepackage{algorithm}
\usepackage{algorithmic}
\usepackage{xspace}
\usepackage{enumitem}
\usepackage[color=white]{todonotes}

\newcommand{\name}{LOT\xspace}

\newcommand{\ot}{\name\xspace}

\newcommand{\fft}{\text{FFT}}

\title{\name: Layer-wise Orthogonal Training on Improving $\ell_2$ Certified Robustness}
\author{Xiaojun Xu~~~~~~~~Linyi Li~~~~~~~~Bo Li
\\
University of Illinois Urbana-Champaign 
\\
\texttt{\{\href{mailto:xiaojun3@illinois.edu}{xiaojun3},
\href{mailto:linyi2@illinois.edu}{linyi2},  \href{mailto:lbo@illinois.edu}{lbo}\}@illinois.edu}
}

\begin{document}

\maketitle

\begin{abstract}
    Recent studies show that training deep neural networks (DNNs) with Lipschitz constraints are able to enhance adversarial robustness and other model properties such as stability.
    In this paper, we propose a layer-wise orthogonal training method (\name) to effectively train 1-Lipschitz convolution layers via parametrizing an orthogonal matrix with an unconstrained matrix. We then efficiently compute 
    the inverse square root of a convolution kernel by transforming the input domain to the Fourier frequency domain.
    On the other hand, as existing works show that semi-supervised training helps improve \textit{empirical} robustness, we aim to bridge the gap and prove that semi-supervised learning  also improves the \textit{certified} robustness of Lipschitz-bounded models.
    We conduct comprehensive evaluations for \name under different settings. We show that \name significantly outperforms baselines regarding deterministic $\ell_2$ certified robustness, and scales to deeper neural networks.
    Under the supervised scenario, we improve the state-of-the-art certified robustness for all architectures (e.g. from 59.04\% to 63.50\% on CIFAR-10 and from 32.57\% to 34.59\% on CIFAR-100 at radius $\rho=36/255$ for 40-layer networks).
    With semi-supervised learning over unlabelled data, we are able to improve state-of-the-art certified robustness on CIFAR-10 at $\rho=108/255$ from 36.04\% to 42.39\%.
    In addition, \name consistently outperforms baselines on different model architectures with only 1/3 evaluation time.
     
\end{abstract}


\section{Introduction}
Given the wide applications of deep neural networks (DNNs), ensuring their robustness against potential adversarial attacks~\cite{goodfellow2014explaining,qiu2020semanticadv,xiao2018generating,xiao2018spatially} is of great importance. There has been a line of research providing defense approaches to improve the \textit{empirical} robustness of DNNs~\cite{madry2017towards,yang2021trs,xiao2018characterizing,xiao2019advit}, and certification methods to \textit{certify} DNN robustness~\cite{li2020sok,YangImp2022,li2022dsrs,li2021tss,yang2022on}. Existing certification techniques can be categorized as deterministic and probabilistic certifications~\cite{li2020sok}, and in this work we will focus on improving the deterministic $\ell_2$ certified robustness by training 1-Lipschitz DNNs.

Although different approaches have been proposed to empirically enforce the Lipschitz constant of the trained model~\cite{tsuzuku2018lipschitz},  it is still challenging to strictly ensure the 1-Lipschitz, which can lead to tight robustness certification.
One recent work SOC~\cite{singla2021skew} proposes to parametrize the orthogonal weight matrices with the exponential of skew-symmetric matrices (i.e. $W = \exp(V-V^\intercal)$). However, such parametrization will be biased when the matrix norm is constrained and the expressiveness is limited, especially when they rescale $V$ to be small to help with convergence.
In this work, we propose a layer-wise orthogonal training approach (\name) by parameterizing the orthogonal weight matrix with an unconstrained matrix $W = (V V^\intercal)^{-\frac12} V$. In order to calculate the inverse square root for convolution kernels, we will perform Fourier Transformation and calculate the inverse square root of matrices on the frequency domain using Newton's iteration. In our parametrization, the output is agnostic to the input norm (i.e. scaling $V=\alpha V$ does not change the value of $W$).
We show that such parametrization achieves higher model expressiveness and robustness (Section~\ref{sec:exp-sup}), and provides more meaningful representation vectors (Section~\ref{sec:exp-abl}). 

In addition, several works have shown that semi-supervised learning will help improve the empirical robustness of models under adversarial training and other settings~\cite{carmon2019unlabeled}.
In this work, we take the first attempt to bridge semi-supervised training with \textit{certified} robustness based on our 1-Lipschitz DNNs.
Theoretically, we show that  semi-supervised learning can help improve the error bound of Lipschitz-bounded models. We also lower bound the certified radius as a function of the model performance and Lipschitz property.
Empirically, we indeed observe that including un-labelled data will help with the certified robustness of 1-Lipschitz models, especially at a larger radius (e.g. from 36.04\% to 42.39\% at $\rho=108/255$ on CIFAR-10).

We conduct comprehensive experiments to evaluate our approach, and we show that \name significantly outperforms the state-of-the-art in terms of the deterministic $\ell_2$ certified robustness.
We also conduct different ablation studies to show that (1) \name can produce more meaningful features for visualization; (2) residual connections help to smoothify the training of \name model.

\underline{\textbf{Technical contributions}}. In this work, we aim to train a certifiably robust 1-Lipschitz model and also analyze the certified radius of the Lipschitz bounded model under semi-supervised learning. 
\begin{itemize}[leftmargin=*]
\vspace{-1em}
\item We propose a layer-wise orthogonal training method \name for convolution layers to train 1-Lipschitz models based on Newton's iteration, and thus compute the deterministic certified robustness for the model. We prove the convergence of Newton's iteration used in our algorithm.
\item We derive the certified robustness of lipschitz constrained model under semi-supervised setting, and formally show  how semi-supervised learning affects the certified radius.
\item We evaluate our \name method under different settings (i.e. supervised and semi-supervised) on different models and datasets. With supervised learning, we show that it significantly outperforms  state-of-the-art baselines, and on some deep architectures the performance gain is over 4\%. With semi-supervised learning, we further improve certified robust accuracy by over 6\% at a large radius.
\end{itemize}

\vspace{-1em}
\section{Related Work}
\vspace{-1em}
\paragraph{Certified Robustness for Lipschitz Constrained Models}
Several studies have been conducted to explore the Lipschitz-constrained models for certified robustness. \cite{tsuzuku2018lipschitz} first certifies model robustness based on its Lipschitz constant and propose training algorithms to regularize the model Lipschitz. Multiple works~\cite{cisse2017parseval,miyato2018spectral, qian2018l2,gouk2021regularisation} have been proposed to achieve 1-Lipschitz during training for linear networks by regularizing or normalizing the spectral norm of the weight matrix. However, when applying these approaches on convolution layers, the spectral norm is bounded by unrolling the convolution into linear operations, which leads to a loose Lipschitz bound~\cite{wang2020orthogonal}. Recently, \cite{anil2019sorting} shows that the 1-Lipschitz requirement is not enough for a good robust model; rather, the gradient-norm-preserving property is important. Besides these training-time techniques, different approaches have been proposed to calculate a tight Lipschitz bound during evaluation. \cite{fazlyab2019efficient} upper bounds the Lipschitz with semi-definite programming while \cite{latorre2020lipschitz} upper bounds the Lipschitz with polynomial optimization. In this work we aim to effectively train 1-Lipschitz convolution models.

\vspace{-1em}
\paragraph{Orthogonal Convolution Neural Networks}
\cite{li2019preventing} first proposes to directly construct orthogonal convolution operations. Such operations are not only 1-Lipschitz, but also gradient norm preserving, which provides a higher model capability and a smoother training process~\cite{anil2019sorting}. BCOP~\cite{li2019preventing} trains orthogonal convolution by iteratively generating $2\times 2$ orthogonal kernels from orthogonal matrices. \cite{trockman2021orthogonalizing} proposes to parametrize an orthogonal convolution with Cayley transformation $W = (I-V+V^\intercal)(I+V-V^\intercal)^{-1}$ where the convolution inverse is calculated on the Fourier frequency domain. ECO~\cite{yu2021constructing} proposes to explicitly normalize all the singular values~\cite{sedghi2018singular} of convolution operations to be 1. So far, the best-performing orthogonal convolution approach is  SOC~\cite{singla2021skew}, where they parametrize the orthogonal convolution with $W=\exp (V-V^\intercal)$ where the exponential and transpose are defined with the convolution operation. However, one major weakness of SOC is that it will rescale $V$ to be small so that the approximation of $\exp$ can converge soon, which will impose a bias on the resulting output space. For example, when $V$ is very small $W$ will be close to $I$. Such norm-dependent property is not desired, and thus we propose a parametrization that is invariant to rescaling. Finally, \cite{singla2021improved} proposes several techniques for orthogonal CNNs, including a generalized Householder (HH) activation function, a certificate regularizer (CReg) loss, and a last layer normalization (LLN). They can be integrated with our training approach to improve model robustness.

\section{Problem Setup}
\subsection{Lipschitz Constant of Neural Networks and Certified Robustness}
Let $f:\mathbb{R}^d \rightarrow \mathbb{R}^C$ denote a neural network for classification, where $d$ is the input dimension and $C$ is the number of output classes. The model prediction is given by $\argmax_c f(x)_c$, where $f(x)_c$ represents the prediction probability for class $c$. The Lipschitz constant of the model under $p$-norm is defined as: $\text{Lip}_p(f) = \sup \frac{||f(x_1)-f(x_2)||_p}{||x_1-x_2||_p}\quad \forall x_1,x_2 \in \mathbb{R}^d$.
Unless specified, we will focus on $\ell_2$-norm and use $\text{Lip}(f)$ to denote $\text{Lip}_2(f)$ in this work. We can observe that the definition of model Lipschitz aligns with its robustness property - both require the model not to change much w.r.t. input changes.
Formally speaking, define $\mathcal{M}_f(x)=\max_i f(x)_i - \max_{j \neq \argmax_i f(x)_i} f(x)_j$ to be the prediction gap of $f$ on the input $x$, then we can guarantee that $f(x)$ will not change its prediction within $|x'-x|<r$, where
\begin{align*}
    r=\mathcal{M}_f(x) / (\sqrt{2} \text{Lip}(f)).
\end{align*}
Therefore, people have proposed to utilize the model Lipschitz to provide certified robustness and investigated training algorithms to train a model with a small Lipschitz constant.

Note that the Lipschitz constant of a composed function $f=f_1 \circ f_2$ satisfies $\text{Lip}(f) \leq \text{Lip}(f_1) \times \text{Lip}(f_2)$. Since a neural network is usually composed of several layers, we can investigate the Lipschitz of each layer to calculate the final upper bound. If we can restrict each layer to be 1-Lipschitz, then the overall model will be 1-Lipschitz with an arbitrary number of layers.

\subsection{Orthogonal Linear and Convolution Operations} Consider a linear operation with equal input and output dimensions $y=Wx$, where $x,y \in \mathbb{R}^n$ and $W \in \mathbb{R}^{n\times n}$. We say $W$ is an orthogonal matrix if $WW^\intercal=W^\intercal W = I$ and call $y=Wx$ an orthogonal linear operation. The orthogonal operation is not only 1-Lipschitz, but also norm-preserving, i.e., $||Wx||_2 = ||x||_2$ for all $x \in \mathbb{R}^n$. If the input and output dimensions do not match, i.e., $W \in \mathbb{R}^{m \times n}$ where $n \neq m$,  we say $W$ is semi-orthogonal if either $WW^\intercal=I$ or $W^\intercal W=I$. The semi-orthogonal operation is 1-Lipschitz and non-expansive, i.e., $||Wx||_2 \leq ||x||_2$.

The orthogonal convolution operation is defined in a similar way. Let $y=W \circ X$ be an orthogonal operation, where $x,y \in \mathbb{R}^{c\times w\times w}$ and $W \in \mathbb{R}^{c\times c\times k \times k}$. We say $W$ is an orthogonal convolution kernel if $W\circ W^\intercal=W^\intercal \circ W = I_{conv}$ where the transpose here refers to the convolution transpose and $I_{conv}$ denotes the identity convolution kernel. Such orthogonal convolution is 1-Lipschitz and norm-preserving. When the input and output channel numbers are different, the semi-orthogonal convolution kernel $W$ satisfies either $W\circ W^\intercal = I_{conv}$ or $W^\intercal \circ W = I_{conv}$ and it is 1-Lipschitz and non-expansive.

\section{\name: Layer-wise Orthogonal Training}

In this section, we propose our \name framework to achieve certified robustness.\footnote{The code is available at \url{https://github.com/AI-secure/Layerwise-Orthogonal-Training}.} We will first introduce how \name layer works to achieve 1-Lipschitz. The key idea of our method is to parametrize an orthogonal convolution $W$ with an un-constrained convolution kernel $W = (V V^\intercal)^{-\frac12} V$. Next, we will propose several techniques to improve the training and evaluation processes of our model. Finally, we discuss how semi-supervised learning can help with the certified robustness of our model.
\subsection{1-Lipschitz Neural Networks via \name}
    \label{subsec:method-ot}
Our key observation is that we can parametrize an orthogonal matrix $W \in \mathbb{R}^{n\times n}$ with an unconstrained matrix $V \in \mathbb{R}^{n\times n}$ by $W=(V V^\intercal)^{-\frac12} V$~\cite{huang2020controllable}.
In addition, this equation also holds true in the case of convolution kernel - given any convolution kernel $V \in \mathbb{R}^{c\times c \times k \times k}$, where $c$ denotes the channel number and $k$ denotes the kernel size, we can get an orthogonal convolution kernel by $W=(V\circ V^\intercal)^{-\frac12}\circ V$, where transpose and inverse square root are with respect to the convolution operations. The orthogonality of $W$ can be verified by $W \circ W^\intercal = W^\intercal \circ W = I_{conv}$, where $I$ is the identity convolution kernel. This way, we can parametrize an orthogonal convolution layer by training over the un-constrained parameter $V$.

Formally speaking, we will parametrize an orthogonal convolution layer by $Y=(V\circ V^\intercal)^{-\frac12}\circ V \circ X$, where $X,Y \in \mathbb{R}^{c \times w \times w}$ and $w\geq k$ is the input shape. The key obstacle here is how to calculate the inverse square root of a convolution kernel. Inspired by \cite{trockman2021orthogonalizing}, we can leverage the convolution theorem which states that the convolution in the input domain equals the element-wise multiplication in the Fourier frequency domain. In the case of multi-channel convolution, the convolution corresponds to the matrix multiplication on each pixel location. That is, let $\fft:\mathbb{R}^{w \times w} \rightarrow \mathbb{C}^{w\times w}$ be the 2D Discrete Fourier Transformation and $\fft^{-1}$ be the inverse Fourier Transformation. We will zero-pad the input to $w\times w$ if the original shape is smaller than $w$. Let $\tilde{X}_{i}=\fft(X_{i})$ and $\tilde{V}_{j,i}=\fft(V_{j,i})$ denote the input and kernel on  frequency domain, then we have:
\begin{align*}
    \fft(Y)_{:,a,b} = (\tilde{V}_{:,:,a,b} \tilde{V}_{:,:,a,b}^*)^{-\frac12} \tilde{V}_{:,:,a,b} \tilde{X}_{:,a,b}
\end{align*}
in which multiplication, transpose and inverse square root operations are matrix-wise. Therefore, we can first calculate $\fft(Y)$ on the frequency domain and perform the inverse Fourier transformation to get the final output.

We will use Newton's iteration  to calculate the inverse square root of the positive semi-definite matrix $A=\tilde{V}_{:,:,a,b} \tilde{V}_{:,:,a,b}^*$ in a differentiable way~\cite{lin2017improved}. If we initialize $Y_0=A$ and $Z_0=I$ and perform the following update rule:
\begin{equation}
    Y_{k+1} =\frac12 Y_k (3I-Z_k Y_k), \quad
    Z_{k+1} =\frac12 (3I-Z_k Y_k) Z_k, \\
\label{eq:newton-iter}
\end{equation}
$Z_k$ will converge to $A^{-\frac12}$ when $||I-A||_2 < 1$. The condition can be satisfied by rescaling the parameter $V=\alpha V$, noticing that the scaling factor will not change the resulting $W$.
In practice, we execute a finite number of Newton's iteration steps and we provide a rigorous error bound for this finite iteration scheme in \Cref{subsec:newton-error-control} to show that the error will decay exponentially.
{
In addition, we show that although the operation is over complex number domain after the FFT, the resulting parameters will still be real domain, as shown below.

\begin{adxtheorem}
    Say $J \in \mathbb{C}^{m \times m}$ is unitary so that $J^* J = I$, and $V = J\tilde{V} J^*$ for $V \in \mathbb{R}^{m \times m}$ and $\tilde{V}\in \mathbb{C}^{m \times m}$. Let $F(V) = (V V^*)^{-\frac12} V$ be our transformation. Then $F(V) = J F(\tilde{V}) J^*$.
\end{adxtheorem}

\begin{proof}
    First, notice that $V^T = V^* = J \tilde{V}^* J^*$. Second, we have:
    \begin{align*}
        (\tilde{V} \tilde{V}^*)^{-1} &= J^* (J \tilde{V} \tilde{V}^* J^*)^{-1} J \\
        &= (J^* (J \tilde{V} \tilde{V}^* J^*)^{-\frac12}J)^2,
    \end{align*}
    so that 
    $(\tilde{V} \tilde{V}^*)^{-\frac12} = J^* (J \tilde{V} \tilde{V}^* J^*)^{-\frac12}J$.
    Thus, we have
    \begin{align*}
        J^* F(V) J &= J^* (V V^T)^{-\frac12} V J \\
        & 
        \overset{\text{by }V=J\tilde{V}J^*}{=} J^* (J\tilde{V}J^* J\tilde{V}^*J^*)^{-\frac12}J\tilde{V}J^*J = J^* (J\tilde{V}\tilde{V}^*J^*)^{-\frac12}J\tilde {V}
        \\
        &= (\tilde{V} \tilde{V}^*)^{-\frac12} \tilde{V}.
    \end{align*}
\end{proof}
\begin{remark}
    From this theorem it is clear that our returned value $J F(\tilde{V}) J^*$ equals to the transformed version of the original matrix $F(V) \in \mathbb{R}^{m\times m}$, and thus is guaranteed to be in the real domain.
\end{remark}
}

\paragraph{Circular Padding vs. Zero Padding} When we apply the convolution theorem to calculate the convolution $Y=W\circ X$ with Fourier Transformation, the result implicitly uses circular padding. However, in neural networks, zero padding is usually a better choice. Therefore, we will first perform zero padding on both sides of input $X^{pad}=\text{zero\_pad}(X)$ and calculate the resulting $Y^{pad}=W\circ X^{pad}$ with Fourier Transformation. Thus, the implicit circular padding in this process will actually pad the zeros which we padded beforehand. Finally, we truncate the padded part and get our desired output $Y = \text{truncate}(Y^{pad})$. We empirically observe that this technique helps improve the model robustness as shown in \Cref{sec:app-circ-pad}.

\paragraph{When Input and Output Dimensions Differ} In previous discussion, we assume that the convolution kernel $V$ has equal input and output channels. In the case $W\in \mathbb{R}^{c_{out}\times c_{in} \times k \times k}$ where $c_{out} \neq c_{in}$, we aim to get a semi-orthogonal convolution kernel $W$ (i.e. $W\circ W^\intercal = I_{conv}$ if $c_{out} < c_{in}$ or  $W^\intercal\circ W = I_{conv}$ if $c_{out} > c_{in}$). As pointed out in \cite{huang2020controllable}, calculating $W$ with Newton's iteration will naturally lead to a semi-orthogonal convolution kernel when the input and output dimensions differ.

\paragraph{Emulating a Larger Stride}
To emulate the case when the convolution stride is 2, we follow previous works~\cite{singla2021skew} and use the invertible downsampling layer~\cite{jacobsen2018revnet}, in which the input dimension $c_{in}$ will be increased by $\times 4$ times.
Strides larger than 2 can be emulated with similar techniques if needed.

\paragraph{Overall Algorithm}
Taking the techniques we discussed before, we can get our final \name convolution layer. The detailed algorithm is shown in Appendix~\ref{sec:app-alg}.
First, we will pad the input to prevent the implicit circular padding mechanism and pad the kernel so that they are in the same shape. Next, we perform the Fourier transformation and calculate the output on the frequency domain with Newton's iteration.
Finally, we perform inverse Fourier transformation and return the desired output.

\paragraph{Comparison with SOC}
Several works have been proposed on orthogonal convolution layers with re-parametrization\cite{trockman2021orthogonalizing,singla2021skew}, among which the SOC approach via $W=\exp (V - V^\intercal)$ has achieved the best performance. Compared with SOC, \name has the following advantages. \underline{First}, the parametrization in \name is norm-independent. Rescaling $V$ to $\alpha V$ will not change the resulting $W$. By comparison, $V$ with a smaller norm in SOC will lead to a $W$ closer to the identity transformation. Considering that the norm of $V$ will be regularized during training (e.g. SOC will re-scale $V$ to have a small norm; people usually initialize weight to be small and impose l2-regularization during training), the orthogonal weight space in SOC may be biased. \underline{Second}, we can see that \name is able to model any orthogonal kernel $W$ by noticing that $(W W^\intercal)^{-\frac12}W$ is $W$ itself; by comparison, SOC cannot parametrize all the orthogonal operations. For example, in the case of orthogonal matrices, the exponential of a skew-symmetric matrix only models the special orthogonal group (i.e. the matrices with +1 determinant). \underline{Third}, we directly handle the case when $c_{in} \neq c_{out}$, while SOC needs to do extra padding so that the channel numbers match. \underline{Finally}, \name is more efficient during evaluation, when we only need to perform the Fourier and inverse Fourier transformation as extra overhead, while  SOC needs multiple convolution operations to calculate the exponential. We will further show quantitative comparisons with SOC in \Cref{sec:exp-sup}.

The major limitation of \name is its large overhead during training, since we need to calculate Newton's iteration in each training step which takes more time and memory.
In addition, we sometimes observe that Newton's iteration is not stable when we perform many steps with 32-bit precision. To overcome this problem, we will pre-calculate Newton's iteration with 64-bit precision during evaluation, as we will introduce in Section~\ref{sec:imp-detail}.

\subsection{Training and Evaluation of \name}
\label{sec:imp-detail}

\paragraph{Smoothing the Training Stage}
In practice, we observe that the \name layers are highly non-smooth with respect to the parameter $V$ (see Section~\ref{sec:exp-abl}). Therefore, the model is difficult to converge during the training process especially when the model is deep. To smooth the training, we propose two techniques. First, we will initialize all except bottleneck layers where $c_{in}=c_{out}$ with identity parameter $V=I$. The bottleneck layers where $c_{in}=c_{out}$ will still be randomly initialized. Second, as pointed out in \cite{li2018visualizing}, residual connection helps with model smoothness. Therefore, for the intermediate layers, we will add the 1-Lipschitz residual connection $y=\lambda x + (1-\lambda) f(x)$. Some work suggests that $\lambda$ here can be trainable~\cite{trockman2021orthogonalizing}, while we observe that setting $\lambda=0.5$ is enough.

\paragraph{Speeding up the Evaluation Stage}
Notice that after the model is trained, the orthogonal kernel $\tilde{W}$ will no longer change. Therefore, we can pre-calculate Newton's iteration and use the result for each evaluation step. Thus, the only runtime overhead compared with a standard convolution layer during evaluation is the Fourier transformation part. In addition, we observe that using 64-bit precision instead of the commonly used 32-bit precision in Newton's iteration will help with the numerical stability. Therefore, when we pre-calculate $\tilde{W}$, we will first transform $\tilde{V}$ into float64. After Newton's iterations, we transform the resulting $\tilde{W}$ back to float32 type for efficiency.


\subsection{Semi-supervised Learning}
Existing works have shown that semi-supervised learning helps improve empirical robustness~\cite{carmon2019unlabeled}, while the impact of semi-supervised learning on certified robustness has not been explored yet. We aim to bridge this gap and investigate whether applying semi-supervised learning also benefits the certified robustness of Lipschitz constrained models. As we will discuss in Section~\ref{sec-theory-cert}, we  theoretically show that the certified robustness of a Lipschitz-bounded model can be improved by semi-supervised learning. In practice, suppose we have a (small) labelled dataset $\mathcal{D}_s=\{(x_i,y_i)\}$ and a (large) unlabelled dataset $\mathcal{D}_{u}=\{x'_i\}$, we will first train a labeller model $G_{pl}$ over the labelled dataset $\mathcal{D}_s$. Then we use $G_{pl}$ to assign pseudo-labels to the unlabelled dataset to get $\mathcal{D'}_{u}=\{(x'_i,G_{pl}(x'_i))\}$. Finally, we use the overall dataset $\mathcal{D}_{s}+\mathcal{D'}_{u}$ to train the student model $G$ as the final model. We assume that both $G_{pl}$ and $G$ will use the \name model architecture.

\section{Theoretical Analysis}
\label{sec:theory}
In this section, we will introduce our theoretical analysis to show that the certified robustness of Lipschitz-constrained models will benefit from semi-supervised learning. We make similar assumptions with \cite{wei2020theoretical} which provides error bound of discrete models in semi-supervised learning, while we will analyze a (Lipschitz-bounded) continuous model and  its certified accuracy.

We first introduce \textit{data expansion} proposed in \cite{wei2020theoretical}, which we will adopt in our analysis. Let $\mathcal{T}$ denote some set of transformations and define $\mathcal{B}\triangleq \{x' | \exists T \in \mathcal{T} \text{ such that } ||x'-T(x)||<r\}$ to be the set of points within distance $r$ from any of the transformation of $x$. The neighbourhood of $x$ is defined by $\mathcal{N}=\{x'|\mathcal{B}(x) \cap \mathcal{B}(x') \neq \emptyset \}$. Let $P$ denote the data distribution and $P_i$ denote the distribution of data in the $i$-th class. The expansion property is defined as:
\begin{adxdefinition} [$(a,c)$-expansion, as in \cite{wei2020theoretical}]
We say that the class-conditional distribution $P_i$ satisfies $(a,c)$-expansion if for all $V \subseteq \mathcal{X}$ with $P_i(V)<a$, the following holds:
\begin{align*}
    P_i(\mathcal{N}(V)) \geq \min \{cP_i(V), 1\}
\end{align*}
if $P_i$ satisfies $(a,c)$-expansion for all classes in $P$, then we say $P$ satisfies $(a,c)$-expansion.
\end{adxdefinition}

Next, we define the classifier and the loss function in the continuous case. For simplicity, we assume a binary classification task here.
\begin{adxdefinition} [Continuous model $G$ and loss functions]
We assume a binary classification task, where $G^*(x), G_{pl}(x) \in \{0,1\}$ are discrete labels, and the trained model is $G(x) \in [0,1]$. We define the loss function $L_\mathsf{m}(G,G^*)\in[0,1]$ to be:
\begin{align*}
    L_\mathsf{m}(G,G^*) = \mathbb{E}_{x\in\mathcal{X}}[G^*(x)(1-G(x)) + (1-G^*(x))G(x)]
\end{align*}
and the error of $G$ is defined as $\text{Err}_\mathsf{m}(G) = L_\mathsf{m}(G,G^*)$.
\end{adxdefinition}
\begin{remark}
Here we provide a general setting here, and if we assume $G(x) \in \{0,1\}$, $L_\mathsf{m}(G,G^*)$ becomes the 0-1 loss in \cite{wei2020theoretical}.
\end{remark}

Then we define the \textit{marginal} consistency set $S_B^\mathsf{m}(G;\delta)$ and loss $R_B^\mathsf{m}(G;\delta)$ in the continuous case. The consistency set includes cases in which the model prediction will not change a lot within the neighbourhood, and the consistency loss measures the probability that the model is not consistent over the population.
\begin{adxdefinition} [Marginal consistency set and consistency loss]
We define the set $S_B^\mathsf{m}(G;\delta)$ in which the  ratio of marginal prediction change on each input $x$ is no larger than $\delta$ within its neighbourhood:
\begin{align*}
    S_B^\mathsf{m}(G;\delta) = \{x | \frac{G(x)}{G(x')} \geq 1-\delta \text{ and } \frac{1-G(x)}{1-G(x')} \geq 1-\delta \quad \forall x' \in \mathcal{B}(x) \}.
\end{align*}
and $R_B^\mathsf{m}(G;\delta)$ denotes the marginal consistency loss of the probability that $x$ is not in $S_B^\mathsf{m}(G;\delta)$:
\begin{align*}
R_B^\mathsf{m}(G;\delta) = \mathbb{P}_{x\in P} [x \notin S_B^\mathsf{m}(G;\delta)].
\end{align*}
\end{adxdefinition}
\begin{remark}
(1) We may use $S_B^\mathsf{m}(G)$ and $R_B^\mathsf{m}(G)$ for abbreviation when the choice of $\delta$ does not lead to confusion; (2) when $\delta=0$, $S_B^\mathsf{m}(G)$ and $R_B^\mathsf{m}(G)$ requires that the prediction is exactly the same within the neighbourhood, which are reduced to $S_B(G)$ and $R_B(G)$ defined in \cite{wei2020theoretical}. 
\end{remark}

Finally, we define the certified robustness radius $\text{CertR}(G)$ as the average radius over the population in which the model keeps its correct prediction:
\begin{align*}
    \text{CertR}(G) = \mathbb{E}_{x\in\mathcal{X}}\big[ \sup r \text{ s.t.} \forall ||x'-x||_2<r, G(x')=G(x)=G^*(x) \big]
\end{align*}
Note that when $G(x)\neq G^*(x)$, the certified radius $r$ will be 0.

\subsection{Error Bound and Certified Radius of Lipschitz-bounded Model}
\label{sec-theory-cert}
Let $\mathcal{M}(G_{pl}) = \{x:G_{pl}(x) \neq G^*(x)\}$ denote the set in which the pseudolabel is wrong. We will have similar separation and expansion assumptions on the data distribution as in \cite{wei2020theoretical}.
\begin{assumption}[Separation]
\label{thm:separation}
We assume that $P$ is separable with probability $1-\mu$ by ground-truth classifier $G^*$, i.e., $R_B^\mathsf{m}(G^*;0) \leq \mu$.
\end{assumption}
\begin{assumption}[Expansion]
\label{thm:expansion}
Define $\overline{a} \triangleq \max_i \{P_i(\mathcal{M}(G_{pl}))\}$ to be the maximum fraction of incorrectly pseudolabeled examples in any class. We assume that $\overline{a}<1/3$ and $P$ satisfies $(\overline{a},\overline{c})$-expansion for $\overline{c}>3$. We express our bounds in terms of $c \triangleq \min(1/\overline{a},\overline{c})$.
\end{assumption}

With these assumptions and let $\delta \in [0,\frac1c]$, we can compare the performance of $G$ and $G_{pl}$ (i.e. the performance with and without semi-supervised learning) with the following theorem:
\begin{adxtheorem} [Error bound with semi-supervised learning]
\label{thm:main2}
Suppose Assumption~\ref{thm:separation} and \ref{thm:expansion} holds and let $\hat{G}$ be a minimizer of the following loss function:
\begin{align*}
    \hat{G} = \argmin_G L(G) \triangleq \frac{c+3}{c-1} L_{\mathsf{m}} (G,G_{pl})+ \frac{2c+2}{c-1} R_{\mathcal{B}}^\mathsf{m}(G;\delta) - \text{Err}(G_{pl})
\end{align*}
\vspace{-0.5em}
Then we can upper bound the error of $\hat{G}$ by:
\begin{align*}
    \text{Err}_\mathsf{m}(\hat{G}) \leq \frac{4}{c-1} \text{Err}(G_{pl}) + \frac{2c+2}{c-1} \mu
\end{align*}
\end{adxtheorem}
\vspace{-1.0em}
In addition, we can bound the certified robustness radius w.r.t. the Lipschitz constant as below:
\begin{adxtheorem} [Certified radius with semi-supervised learning]
\label{thm:main}
Suppose Assumption~\ref{thm:expansion} holds. Then the certified radius of a model $G$ can be bounded by:
\begin{align*}
    \text{CertR}(G) \geq \frac{0.5-\frac{c+3}{c-1} L_{\mathsf{m}} (G,G_{pl}) - \frac{2c+2}{c-1} R_{\mathcal{B}}^\mathsf{m}(G;\delta) + \text{Err}(G_{pl})}{\text{Lip}(G)}.
\end{align*}
\end{adxtheorem}
The proofs of Theorem~\ref{thm:main2} and \ref{thm:main} are shown in Appendix~\ref{sec:app-proof}. 

\paragraph{Discussion} From Theorem~\ref{thm:main2}, we observe that the benefits of semi-supervised learning depend on the error of $G_{pl}$ and data property. If the data expands and separates well and $G_{pl}$ is accurate,  $c$ is large and $\mu$ is small, which means $\text{Err}_\mathsf{m}(\hat{G})$ will be smaller than $\text{Err}(G_{pl})$. Note that $\mu$ is usually small and we assume $c>3$, so $\text{Err}_\mathsf{m}(\hat{G}) < \text{Err}(G_{pl})$ will hold true in most cases, indicating that semi-supervised learning helps improve the model performance. From Theorem~\ref{thm:main}, we can observe that given the data distribution and $G_{pl}$, the certified radius of the student model $\text{CertR}(G)$ can be lower-bounded by 1) how well $G$ learns from $G_{pl}$ ($L_\mathsf{m}(G,G_{pl})$); 2) how consistent $G$ is ($R_{\mathcal{B}}^\mathsf{m}(G;\delta)$). These two factors also correspond to the losses in the actual training process - the cross-entropy loss measures how well $G$ learns from $G_{pl}$ and the Lipschitz constraint measures local consistency for a Lipschitz-bounded model\footnote{This means that, with smaller Lipschitz, the model is more likely to be consistent within neighbourhood. In addition, the CReg loss will also improve local consistency by maximizing the prediction gap.}.

\section{Evaluation}
In this section, we will evaluate our method \name on different datasets and models compared with the state-of-the-art baselines based on the deterministic $\ell_2$ certified robustness. We show that on both CIFAR-10 and CIFAR-100, \name achieves significantly higher certified accuracy under different radii and settings with more efficient evaluation time.
In the meantime, we conduct a series of ablation studies to analyze the representation power of our \name, the error control of Newton's iterations, and the effects of residual connections. We show that under the semi-supervised setting, the certified accuracy of \name is further improved and outperforms the baselines.

\subsection{Experiment Setup}
\label{sec:exp-setting}

\paragraph{Baseline - SOC with CReg loss, HH activation and LLN technique} In the evaluation, we will mainly compare with SOC~\cite{singla2021skew}, as SOC has been shown to outperform other 1-Lipschitz models in terms of certified robustness. SOC  parameterizes an orthogonal layer with $W=\exp (V-V^\intercal)$ where $V$ is the convolution kernel and the transpose and exponential are with respect to the convolution operation.
Following their setting, we will focus on the CIFAR-10 and CIFAR-100 datasets and the model architecture varies from LipConvnet-5 to LipConvnet-40.

\cite{singla2021improved} proposes three techniques to improve SOC networks
. First, they observe that adding a Certificate Regularization (CReg) loss can help with the certified performance at larger radius: $\ell_{CReg}(x,y) = - \gamma \text{ReLU}\big( (f(x)_y - \max_{i\neq y}f(x)_i) / \sqrt{2} \big)$. Second, they propose a Householder (HH) activation layer, which is a generalization of the standard GroupSort activation. Finally, for tasks with large number of classes (e.g. CIFAR-100), they propose Last Layer Normalization(LLN) to only normalize the final output layer instead of orthogonalize. Since these techniques are agnostic to the type of convolution operation, we will further integrate them with \name for comparison.

\paragraph{Implementation Details} We will train and evaluate the \name network under supervised and unsupervised scenarios. Following previous works~\cite{singla2021skew}, we focus on the CIFAR-10 and CIFAR-100 datasets and the model architecture varies from LipConvnet-5 to LipConvnet-40. In semi-supervised learning, we use the 500K data introduced in \cite{carmon2019unlabeled} as the unlabelled dataset. To train the \name network, we will train the model for 200 epochs using a momentum SGD optimizer with an initial learning rate 0.1 and decay by 0.1 at the 100-th and 150-th epochs. We use Newton's iteration with 10 steps which we observe is enough for convergence (see \Cref{sec:app-err-control}).
When CReg loss is applied, we use $\gamma=0.5$; when HH activation is applied, we use the version of order 1. We add the residual connection with a fixed $\lambda=0.5$ for \name; for SOC, we use their original version, as we observe that residual connections even hurt their performance (see discussions in Section~\ref{sec:exp-abl}). We show the certified accuracy at radius $\rho\in \{\frac{36}{255}, \frac{72}{255}, \frac{108}{255}\}$. For the evaluation time comparison, we  show the runtime taken to do a full pass on the testing set evaluated on an NVIDIA RTX A6000 GPU.

\vspace{-3mm}
\subsection{Supervised Scenario}
\label{sec:exp-sup} \vspace{-0.5em}

\begin{table}[t]
    \centering
     \caption{\small Certified accuracy of 1-Lipschitz model with CReg loss and HH activation on CIFAR-10/100 in supervised setting. The LLN technique is applied on the CIFAR-100 dataset.}
     \resizebox{\textwidth}{!}{

    \begin{tabular}{l l|| c | c c c || c | c c c || c}
        \toprule
        \multirow{3}{*}{\bf Model} & \multirow{3}{*}{\bf\shortstack{Conv.\\ Type}} & \multicolumn{4}{c||}{CIFAR-10} & \multicolumn{4}{c||}{CIFAR-100} & \bf Mean \\
        \cline{3-10}
        & & \bf Vanilla & \multicolumn{3}{c||}{\bf Certified Accuracy at $\rho=$} & \bf Vanilla & \multicolumn{3}{c||}{\bf Certified Accuracy at $\rho=$} & \bf Evaluation \\
        & & \bf Accuracy & 36/255 & 72/255 & 108/255 & \bf Accuracy & 36/255 & 72/255 & 108/255 & \bf Time (sec) \\
        \midrule
        \multirow{2}{*}{\shortstack{LipConvnet-5}} & SOC & 75.31\% & 60.37\% & 45.62\% & 32.38\% & 45.82\% & 32.99\% & 22.48\% & 14.79\%  & 2.285 \\
        & \ot &\bf  76.34\% & \bf 62.07\% & \bf 47.52\% & \bf 33.99\% & \bf 48.38\% & \bf 34.77\% & \bf 23.38\% & \bf 15.44\% & \bf 1.411 \\
        \hline
        \multirow{2}{*}{\shortstack{LipConvnet-10}} & SOC & 76.23\% & 62.57\% & 47.70\% & 34.15\% & 47.07\% & 34.53\% & 23.50\% & 15.66\% & 3.342  \\
        & \ot & \bf 76.50\% & \bf 63.12\% & \bf 48.59\% & \bf 35.65\% & \bf 48.70\% & \bf 34.82\% & \bf 23.86\% & \bf 15.93\% & \bf 1.563  \\
        \hline
        \multirow{2}{*}{\shortstack{LipConvnet-15}} & SOC & 76.39\% & 62.96\% & 48.47\% & 35.47\% & 47.61\% & 34.54\% & 23.16\% & 15.09\% & 4.105 \\
        & \ot & \bf 76.86\% & \bf 63.84\% & \bf 49.18\% & \bf 36.35\% & \bf 48.99\% & \bf 34.90\% & \bf 24.39\% & \bf 16.37\% & \bf 1.627 \\
        \hline
        \multirow{2}{*}{\shortstack{LipConvnet-20}} & SOC & 76.34\% & 62.63\% & 48.69\% & 36.04\% & 47.84\% & 34.77\% & 23.70\% & 15.84\% & 5.142 \\
        & \ot & \bf 77.12\% & \bf 64.30\% & \bf 49.49\% & \bf 36.34\% & \bf 48.81\% & \bf 35.21\% & \bf 24.37\% & \bf 16.23\% & \bf 1.962 \\
        \hline
        \multirow{2}{*}{\shortstack{LipConvnet-25}} & SOC & 75.21\% & 61.98\% & 47.93\% & 34.92\% & 46.87\% & 34.09\% & 23.41\% & 15.61\% & 6.087 \\
        & \ot & \bf 76.83\% & \bf 64.49\% & \bf 49.80\% & \bf 37.32\% & \bf 48.93\% & \bf 35.23\% & \bf 24.33\% & \bf 16.59\% & \bf 2.387 \\
        \hline
        \multirow{2}{*}{\shortstack{LipConvnet-30}} & SOC & 74.23\% & 60.64\% & 46.51\% & 34.08\% & 46.92\% & 34.17\% & 23.21\% & 15.84\% & 6.927 \\
        & \ot & \bf 77.12\% & \bf 64.36\% & \bf 49.98\% & \bf 37.30\% & \bf 49.18\% & \bf 35.54\% & \bf 24.24\% & \bf 16.48\% & \bf 2.755 \\
        \hline
        \multirow{2}{*}{\shortstack{LipConvnet-35}} & SOC & 74.25\% & 61.30\% & 47.60\% & 35.21\% & 46.88\% & 33.64\% & 23.34\% & 15.73\% & 7.870 \\
        & \ot & \bf 76.91\% & \bf 63.55\% & \bf 49.05\% & \bf 36.19\% & \bf 48.25\% & \bf 34.99\% & \bf 24.13\% & \bf 16.25\% & \bf 3.193 \\
        \hline
        \multirow{2}{*}{\shortstack{LipConvnet-40}} & SOC & 72.59\% & 59.04\% & 44.92\% & 32.87\% & 45.03\% & 32.57\% & 22.37\% & 14.76\% & 8.668 \\
        & \ot & \bf 76.75\% & \bf 63.50\% & \bf 49.07\% & \bf 36.06\% & \bf 48.31\% & \bf 34.59\% & \bf 23.70\% & \bf 15.94\% & \bf 3.595 \\
		\bottomrule
    \end{tabular}
    }
    \label{tab:sup-cifar10-crhh}
    \label{tab:sup-cifar100-crhh}
    \vspace{-1.1em}
\end{table}

We show the results of supervised learning on CIFAR-10 and CIFAR-100 in Table~\ref{tab:sup-cifar10-crhh}. We can observe that  \name  significantly outperforms baselines on both vanilla accuracy and certified accuracy for different datasets. In particular, we observe that the improvement is more significant for deeper models. We owe it to the reason that our \name layer has a better expressiveness and unbiased parametrization. 
``Mean evaluation time'' column in Table~\ref{tab:sup-cifar10-crhh} records the evaluation time per instance averaged between CIFAR-10 and CIFAR-100 models. Since architectures of CIFAR-10 and CIFAR-100 models differ only in the last linear year, the evaluation time is almost the same.
We observe that \name has a better efficiency during the evaluation stage, since \name pre-calculates Newton's iteration while SOC calculates the exponential in each iteration.
A similar conclusion can be observed from results without CReg loss, HH activation and LLN, as shown in Table~\ref{tab:sup-cifar10} and \ref{tab:sup-cifar100} in Appendix~\ref{sec:app-cifar100}.

\begin{table}[t]
    \centering
    \caption{Performance comparison of 1-Lipschitz networks with and without residual connections.}
    \label{tab:abl-residual}
    \begin{tabular}{l l c | c | c c c}
        \toprule
        \multirow{2}{*}{\bf Model} & \bf Conv. & \bf Residual & \bf Vanilla & \multicolumn{3}{c}{\bf Certified Accuracy at $\rho=$} \\
        & \bf Type & \bf Connection & \bf Accuracy & 36/255 & 72/255 & 108/255 \\
        \midrule
        \multirow{4}{*}{LipConvnet-20} & \multirow{2}{*}{SOC} & $\times$ & 76.90\% & 61.87\% & 45.79\% & 31.08\% \\
        & & $\checkmark$ & 76.97\% & 61.76\% & 45.46\% & 30.29\% \\
        \cmidrule{2-7}
        & \multirow{2}{*}{\ot} & $\times$ & 77.16\% & 62.14\% & 45.78\% & 30.95\% \\
        & & \checkmark & \bf 77.86\% & \bf 63.54\% & \bf 47.15\% & \bf 32.12\% \\
		\bottomrule
    \end{tabular}
    \vspace{-1em}
\end{table}



\subsection{Ablation Studies}
\label{sec:exp-abl} \vspace{-1em}
In this section, we will perform several ablation studies for \name. Unless specified, we will use the deep model LipConvnet-20 for evaluation. We use the model without CReg loss and HH activation so that we can see the comparison under the standard CNN setting.

\paragraph{Representation Analysis}
Recent works suggest that adversarially robust models will have a good regularization-free feature visualization~\cite{allen2022feature}. We show the visualization of 4 neurons in the last hidden layer of LipConvNet-20 for both SOC and \name in Figure~\ref{fig:visualize}. More figures are shown in Appendix~\ref{sec:app-visualize} The visualization process is as follows: given a chosen neuron, we will start from a random image and take 500 gradient steps to maximize the neuron value with step size 1.0 and decay factor 0.1. We can see that  \name  indeed generates  more meaningful features. We attribute it to the good expressiveness power of our model. This indicates that our approach indeed leads to better representations.
\begin{figure}[t]
    \centering
    \includegraphics[trim={0 7.6cm 0 0},clip, width=0.4\textwidth]{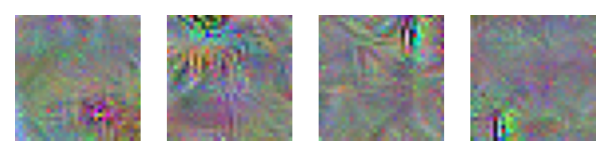}
    \hspace{0.2in}
    \includegraphics[trim={0 7.6cm 0 0},clip,width=0.4\textwidth]{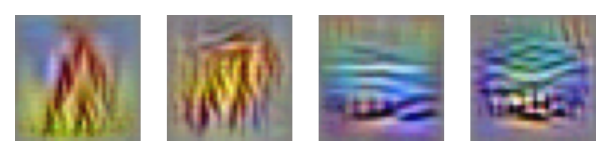}
    \vspace{-1em}
    \caption{\small Visualizing the features in the last hidden layer of LipConvnet-20 for SOC (left) and \name (right). Each image corresponds to one randomly chosen neuron from the last hidden layer and is optimized to maximize the value of the neuron.}
    \label{fig:visualize}
    \vspace{-1.8em}
\end{figure}

\paragraph{Error Control of Newton's Iterations} To see how our Newton's iteration approximates the inverse square root, we visualize the maximum singular value ($\sigma_{max}$) of \name layers in Appendix~\ref{sec:app-err-control}. We observe that all the resulting layers have a strictly $< 1$ Lipschitz. Therefore, we can safely conclude that the overall \name network is 1-Lipschitz.

\paragraph{Effect of Residual Connections}

\begin{wrapfigure}{l}{0.45\textwidth}
\begin{center}
    \captionsetup[subfigure]{justification=centering}
    \begin{subfigure}[b]{0.22\textwidth}
         \centering
         \includegraphics[width=\textwidth]{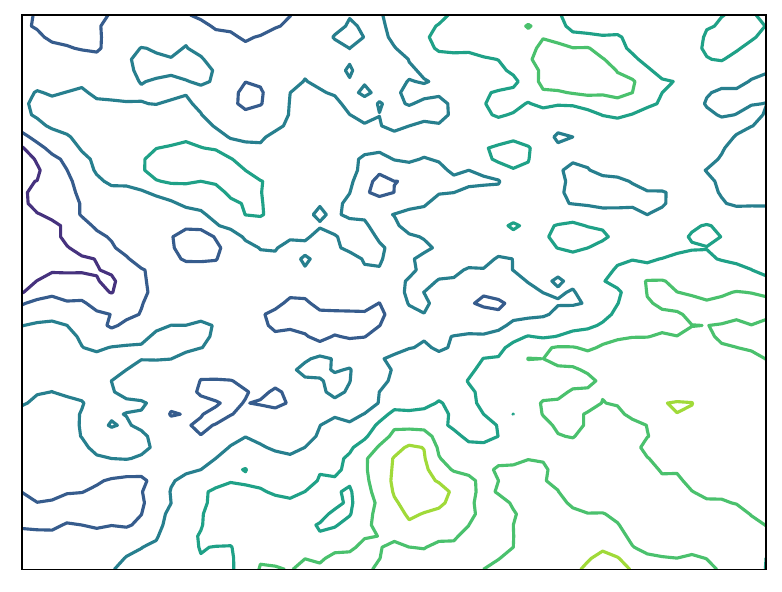}
         \caption{\name, without \\residual connection.}
         \label{fig:contour-ot-no-res}
     \end{subfigure}
    \begin{subfigure}[b]{0.22\textwidth}
         \centering
         \includegraphics[width=\textwidth]{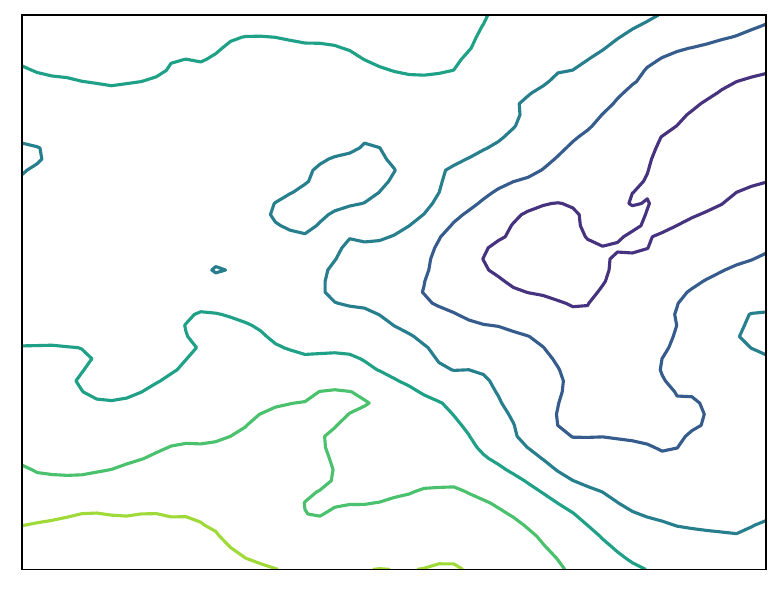}
         \caption{\name, with \\residual connection.}
         \label{fig:contour-ot-res}
     \end{subfigure}
     \caption{\small The loss landscape~\cite{li2018visualizing} with respect to the parameters of \name with and without residual connections. The figure is plotted by calculating the loss contour on two randomly chosen directions of parameters.}
    \label{fig:contour}
\end{center}
\vspace{-2em}
\end{wrapfigure}
Empirically, we observe that \name is non-smooth and therefore we need the residual connection to smoothify the model. To verify this phenomenon, we adopt the visualization approach in \cite{li2018visualizing} and show the loss surface in Figure~\ref{fig:contour}. We can see that \name without residual connection is highly non-smooth, while after adding residual connections the model becomes much smoother. By comparison, we observe that SOC networks with and without residual connection are similar (see Appendix~\ref{sec:app-residual}). We hypothesize the reason to be that the exponential parametrization $W=\exp(A)$ enables residual property implicitly, i.e.,
\begin{align*}
    \exp(A) \circ X = \underline{X + A \circ X} + \frac12 A \circ A \circ X + \ldots
\end{align*}

where the first two terms are essentially the format of residual connection. We show the results of models with and without residual connections in Table~\ref{tab:abl-residual}. We can observe that our \name indeed improves with the residual connection, while the performance remains similar or even worse for SOC. Therefore, in the evaluation, we use the \name with residual connection and SOC without it.

{
\paragraph{Robustness against $\ell_\infty$ Empirical Attack}
To evaluate the robustness of our model against attacks in other $\ell_p$ norms, we evaluate the model against standard $\ell_\infty$ PGD attack\cite{madry2017towards} with $\epsilon=8/255$ and 50 steps. We show the results on CIFAR-10 with CReg loss and HH activation in Table~\ref{tab:emp-cifar10}. We can observe that we still achieve better empirical robustness on the LipConvNet compared with SOC.
}

\begin{table}[tbp]
    \centering
    \caption{Empirical robust accuracy of 1-Lipschitz model with CReg loss and HH activation on CIFAR-10 against $\ell_\infty$-PGD attack at $\epsilon=8/255$. }
    \label{tab:emp-cifar10}
    \begin{tabular}{l l| c c}
        \toprule
        \bf Model & \bf Conv. & \bf Vanilla accuracy & \bf Robust accuracy\\
        \midrule
        \multirow{2}{*}{LipConvnet-5} & SOC & 75.31\% & 27.18\% \\
        & \ot & \bf 76.34\% & \bf 27.19\% \\
        \midrule
        \multirow{2}{*}{LipConvnet-10} & SOC & 76.23\% & 27.95\% \\
        & \ot & \bf 76.50\% & \bf 28.71\% \\
        \midrule
        \multirow{2}{*}{LipConvnet-15} & SOC & 76.39\% & 27.84\% \\
        & \ot & \bf 76.86\% & \bf 29.10\% \\
        \midrule
        \multirow{2}{*}{LipConvnet-20} & SOC & 76.34\% & 26.17\% \\
        & \ot & \bf 77.12\% & \bf 29.45\% \\
        \midrule
        \multirow{2}{*}{LipConvnet-25} & SOC & 75.21\% & 28.93\% \\
        & \ot & \bf 76.83\% & \bf 29.55\% \\
        \midrule
        \multirow{2}{*}{LipConvnet-30} & SOC & 74.23\% & 27.44\% \\
        & \ot & \bf 77.12\% & \bf 28.92\% \\
        \midrule
        \multirow{2}{*}{LipConvnet-35} & SOC & 74.25\% & 14.09\% \\
        & \ot & \bf 76.91\% & \bf 29.32\% \\
        \midrule
        \multirow{2}{*}{LipConvnet-40} & SOC & 72.59\% & 11.68\% \\
        & \ot & \bf 76.75\% & \bf 28.67\% \\
		\bottomrule
    \end{tabular}
\end{table}


\subsection{Semi-supervised Learning Scenario}

\begin{table}[t]
    \centering
    \caption{Performance of 1-Lipschitz networks on the CIFAR-10 dataset with semi-supervised training. Results of different architectures are shown in Table~\ref{tab:ssl-full-cifar10-crhh} in Appendix~\ref{sec:app-semi-sup}.}
    \label{tab:ssl-cifar10-crhh}
    \begin{tabular}{l l l | c | c c c}
        \toprule
        \multirow{2}{*}{\bf Model} & \bf Conv. & \bf \multirow{2}{*}{\bf Setting} & \bf Vanilla & \multicolumn{3}{c}{\bf Certified Accuracy at $\rho=$} \\
        & \bf Type & & \bf Accuracy & 36/255 & 72/255 & 108/255 \\
        \midrule
        \multirow{4}{*}{LipConvnet-20} & \multirow{2}{*}{SOC} & Supervised  & 76.34\% & 62.63\% & 48.69\% & 36.04\% \\
        & & Semi-supervised & 70.95\% & 61.72\% & 51.78\% & 42.01\% \\
        \cmidrule{2-7}
        & \multirow{2}{*}{\ot} & Supervised & \bf 77.12\% & \bf 64.30\% & 49.49\% & 36.34\%  \\
        & & Semi-supervised & 71.86\% & 62.86\% & \bf 52.24\% & \bf 42.39\% \\
		\bottomrule
    \end{tabular}
    \vspace{-1em}
\end{table}

We show the results of semi-supervised learning for one architecture with CReg loss and HH activation in Table~\ref{tab:ssl-cifar10-crhh}. The full results are shown in Table~\ref{tab:ssl-full-cifar10-crhh} and \ref{tab:ssl-full-cifar10} in Appendix~\ref{sec:app-semi-sup}. We can observe that, with semi-supervised learning, the vanilla accuracy will drop slightly, which will impact the certified accuracy at a small radius ($\rho = 36/255$). However, the certified accuracy at a large radius will be improved, and the gap is more significant at a larger radius (e.g. $\rho = 108/255$). We improve the previous state-of-the-art certified accuracy by over 6\% at $\rho=108/255$. This indeed shows that semi-supervised learning can help improve certified robustness. We empirically observe that semi-supervised learning does not help much on CIFAR-100. We owe it to the reason that the vanilla accuracy of the teacher model on CIFAR-100 is low (less than 50\%). Thus, according to Assumption~\ref{thm:expansion} and Theorem~\ref{thm:main}, when $c$ is low, the certified accuracy would be low.



\vspace{-1em}
\section{Conclusion}
\vspace{-1em}
In this work, we propose an orthogonal convolution layer \name and build a 1-Lipschitz convolution network. We show that \name network outperforms the previous state-of-the-art in certified robustness. We also show that semi-supervised learning can further help with the robustness both theoretically and experimentally.

\vspace{-1em}
\section*{Acknowledgements}
\vspace{-1em}
This work is partially supported by the NSF grant No.1910100,
NSF CNS No.2046726, C3 AI, and the Alfred P. Sloan Foundation.

\bibliographystyle{plain}
\bibliography{bib}


\newpage
\appendix

\section{Broader Impact}
\label{sec:broad-imp}
Adversarial attacks have been investigated a lot and people are worried that the vulnerability of machine learning models may affect their application in real world. This paper, as a potential counter-measure against the adversarial attacks, proposes a type of model architecture that can guarantee the model robustness under certain adversarial attacks. This could help to secure the safe deployment of ML models in real applications, thus promoting the development of various ML systems. In addition, better understanding of the Lipschitz property of machine learning models may help people understand and explain how the models work, which is a key concern when applying these algorithms in practice. On the other hand, the misuse of our technique may also lead to negative impact. For example, if an attacker is fully familiar with this work, he/she may discover some specific attack to fool this model (beyond $\ell_2$ attacks which we can certify against). Therefore, in real applications, it is recommended that different defense techniques are applied together to secure the model safety.

\section{Pseudocode of \name Layer}
\label{sec:app-alg}
We show the detailed pseudocode of our \name layer in \Cref{alg:ot-train}.
\begin{algorithm}[H]
\caption{\name layer.}
\label{alg:ot-train}
\begin{algorithmic}[1]
\REQUIRE Unconstrained convolution kernel $V\in \mathbb{R}^{c_{out}\times c_{in} \times k \times k}$; Input tensor $X \in \mathbb{R}^{c_{in} \times w \times w}$. 
\STATE $X^{pad} = \text{zero\_pad}(X, (k,k,k,k)) \in \mathbb{R}^{c_{in} \times (w+2k) \times (w+2k)}$. 
\STATE $V^{pad} = \text{zero\_pad}(W, (0,0,k+w,k+w)) \in \mathbb{R}^{c_{out}\times c_{in} \times (w+2k) \times (w+2k)}$. 
\STATE // Calculate the Fourier transformation:
\FORALL{$i \in \{1,\ldots, c_{in}\}$}
	\STATE $\tilde{X}_{i} = \fft(X_{i}^{pad}).$
	\FORALL {$j \in \{1, \ldots, c_{out}\}$}
	    \STATE $\tilde{V}_{j,i} = \fft(V_{j,i}^{pad}).$
	\ENDFOR
\ENDFOR
\STATE // Calculate the output on frequency domain:
\FORALL{$a,b \in \{1,\ldots, w+2k\}$}
    \STATE $\hat{V} = \frac{\tilde{V}_{:,:,a,b}}{\sqrt{|| \tilde{V}_{:,:,a,b} \tilde{V}_{:,:,a,b}^\intercal ||_F}}$ \quad // Rescale $\tilde{V}$.
    \STATE Calculate $\tilde{W}_{:,:,a,b} = (\hat{V} \hat{V}^*)^{-\frac12} \hat{V}$ with Newton's iteration.
    \STATE $\tilde{Y}_{:,a,b} = \tilde{W}_{:,:,a,b} \tilde{X}_{:,a,b}$.
\ENDFOR
\STATE // Get the final output:
\FORALL{$i \in \{1,\ldots, c_{out}\}$}
	\STATE $Y_{i} = \fft^{-1}(\tilde{Y}_{i}).$
\ENDFOR
\RETURN $(Y_{:,k:w+k,k:w+k}).real$
\end{algorithmic}
\end{algorithm}

\section{Proof of \Cref{thm:main2} and \Cref{thm:main}}
\label{sec:app-proof}

\subsection{Proof}
To prove \Cref{thm:main2} and \Cref{thm:main}, we first define the loss on a subset, so that we can divide the loss into different subset losses.
\begin{adxdefinition} [Subset loss]
Define conditional loss:
\begin{align*}
    L_\mathsf{m}^{cond}(G,G^*|S)=E_{x\in S}[G^*(x)(1-G(x)) + (1-G^*(x))G(x)]
\end{align*}
to bethe loss function calculated only over the subset $S\subseteq \mathcal{X}$, and define the subset loss:
\begin{align*}
    L_\mathsf{m}(G,G^*|S) = L_\mathsf{m}^{cond}(G,G^*|S) \cdot P(S).
\end{align*}
A property of the subset loss is that, if $S=S_1\cup S_2$ where $S_1$ and $S_2$ are disjoint, then $L_\mathsf{m}(G,G^*|S)=L_\mathsf{m}(G,G^*|S_1)+L_\mathsf{m}(G,G^*|S_2)$.
\end{adxdefinition}

First, we provide some facts on the relationship between subset loss $L_\mathsf{m}(G,G^*|S)$ and $L_\mathsf{m}(G,G_{pl}|S)$ when $S \subseteq \mathcal{M}(G_{pl})$ or $S \subseteq \overline{\mathcal{M}(G_{pl})}$.
\begin{fact}
When $S \subseteq \overline{\mathcal{M}(G_{pl})}$, then $L_\mathsf{m}(G,G^*|S) = L_\mathsf{m}(G,G_{pl}|S)$; when $S \subseteq \mathcal{M}(G_{pl})$, then $L_\mathsf{m}(G,G^*|S) + L_\mathsf{m}(G,G_{pl}|S) = P(S)$.
\end{fact}
\begin{proof}
This is easy to see by noticing that when $x \in \overline{\mathcal{M}(G_{pl})}$, then $G_{pl}(x) = G^*(x)$; when $x \in \mathcal{M}(G_{pl})$, then $G_{pl}(x) \neq G^*(x)$, so $G_{pl}(x) + G^*(x) = 1$.
\end{proof}
Now, we define $S_1 = S_B^\mathsf{m}(G) \cap \mathcal{M}(G_{pl})$ and $S_2 = S_B^\mathsf{m}(G) \cap \overline{\mathcal{M}(G_{pl})}$. Base on the fact, we will be able to derive the relationship between $L_\mathsf{m}(G,G^*|S_1)$ and $L_\mathsf{m}(G,G^*|S_2)$:
\begin{lemma}
\label{lemma:relation-of-loss}
We have the following relationship between the losses on the sets in which $G_{pl}$ is correct vs. $G_{pl}$ is wrong:
\begin{align*}
    L_{\mathsf{m}}(G,G^*|S_2) + P(S_1) = L_{\mathsf{m}}(G,G^*|S_1) + L_{\mathsf{m}}(G,G_{pl}|S_B^\mathsf{m}(G))
\end{align*}
And therefore,
\begin{align*}
    L_{\mathsf{m}}(G,G^*|S_2) \leq L_{\mathsf{m}}(G,G^*|S_1) + L_{\mathsf{m}}(G,G_{pl}) - \text{Err}(G_{pl}) + R_B^\mathsf{m}(G)
\end{align*}
\end{lemma}
\begin{proof}
For the first equation, note that $S_2 \subseteq \overline{\mathcal{M}(G_{pl})}$ and $S_1 \subseteq \mathcal{M}(G_{pl})$, so based on the previous fact, we have:
\begin{align*}
    &L_{\mathsf{m}}(G,G^*|S_2) + P(S_1)\\
    =& L_{\mathsf{m}}(G,G^*|S_2) + L_{\mathsf{m}}(G,G_{pl}|S_1) + L_{\mathsf{m}}(G,G^*|S_1)\\
    =& L_{\mathsf{m}}(G,G_{pl}|S_2) + L_{\mathsf{m}}(G,G_{pl}|S_1) + L_{\mathsf{m}}(G,G^*|S_1)\\
    =& L_{\mathsf{m}}(G,G_{pl}|S_B^\mathsf{m}(G)) + L_{\mathsf{m}}(G,G^*|S_1)
\end{align*}

For the second inequation, notice $\mathcal{M}(G_{pl}) \backslash \overline{S_B^\mathsf{m}(G)} \subseteq S_1$, so $P(S_1) \geq P(\mathcal{M}(G_{pl})) - P(\overline{S_B^\mathsf{m}(G)}) = \text{Err}(G_{pl}) - R_B^\mathsf{m}(G)$. So:
\begin{align*}
    L_{\mathsf{m}}(G,G^*|S_2) =& L_{\mathsf{m}}(G,G^*|S_1) + L_{\mathsf{m}}(G,G_{pl}|S_B^\mathsf{m}(G)) - P(S_1) \\
    \leq & L_{\mathsf{m}}(G,G^*|S_1) + L_{\mathsf{m}}(G,G_{pl}) - (\text{Err}(G_{pl}) - R_B^\mathsf{m}(G)) \\
    = & L_{\mathsf{m}}(G,G^*|S_1) + L_{\mathsf{m}}(G,G_{pl}) - \text{Err}(G_{pl}) + R_B(G)
\end{align*}
\end{proof}

\begin{fact}
\label{fact:beta-ineq}
Given $\delta \in [0,\frac1c]$ and $\beta \in (0,\frac{c-1}2]$, we can verify that:
\begin{align*}
    \frac{\beta-1}{c(1-\delta)-2} \leq \frac{\beta}{c-1}
\end{align*}
\end{fact}
\begin{proof}
This can be verified by substituting $\delta = \frac{1}{c}$ into LHS, noticing that LHS is monotonically increasing w.r.t. $\delta$.
\end{proof}

\begin{lemma}
\label{lemma:main-bound}
For any $\beta\in (0,\frac{c-1}2]$, define $q=\frac{\beta}{c-1}\text{Err}(G_{pl})$ and $\alpha=(\beta-1)\text{Err}(G_{pl})$. If $G$ fits the pseudolabels with suffcient accuracy and consistency:
\begin{align*}
    L_{\mathsf{m}}(G,G_{pl}) + 2R_B^\mathsf{m}(G) \leq \text{Err}(G_{pl}) + \alpha
\end{align*}
Then $G$ satisfies the following error bound:
\begin{align*}
    \text{Err}_\mathsf{m}(G) \leq 2(q+R_B^\mathsf{m}(G)) + L_{\mathsf{m}}(G,G_{pl}) - \text{Err}(G_{pl})
\end{align*}
\end{lemma}
The intuition of the proof is as follows. Lemma~\ref{lemma:relation-of-loss} provides a relationship between the loss $L_{\mathsf{m}}(G,G^*|S_1)$ and $L_{\mathsf{m}}(G,G^*|S_2)$. On the other hand, the expansion of $S_1$ is also related with $S_2$ by $(\mathcal{N}(S_1)\backslash S_1) \cap S_B^\mathsf{m}(G) \subseteq S_2$. Note that the expansion $P(\mathcal{N}(S_1)) > c \cdot P(S_1)$, so $S_1$ cannot be too large or otherwise $\mathcal{N}(S_1)\backslash S_1$ will be too large to be within $S_2$. We will show that $L_{\mathsf{m}}(G,G^*|S_1) < q$.
\begin{proof}
Consider the expansion of $S_1$, $\mathcal{N}(S_1)$. Since $S_1 \subseteq \mathcal{M}(G_{pl})$ and $P(G_{pl}) < 1/c$, we know that $c\cdot P(S_1) < 1$, so by the assumption of expansion, $P(\mathcal{N}(S_1)) \geq c\cdot P(S_1)$. In addition, notice that $(\mathcal{N}(S_1)\backslash S_1) \cap S_B^\mathsf{m}(G) \subseteq S_2$. Therefore, we have:
\begin{align*}
    L_{\mathsf{m}}(G,G^*|S_2) \geq & L_{\mathsf{m}}(G,G^*|(\mathcal{N}(S_1)\backslash S_1) \cap S_B^\mathsf{m}(G)) \\
    \geq & L_{\mathsf{m}}(G,G^*|\mathcal{N}(S_1) \cap S_B^\mathsf{m}(G)) - L_{\mathsf{m}}(G,G^*|S_1)
\end{align*}
For the first term, we notice that $\mathcal{N}(S_1) \cap S_B^\mathsf{m}(G) \subseteq \mathcal{N}(S_1)$ and $S_1 \subseteq S_B^\mathsf{m}(G)$, so the conditional loss satisfies $L(G,G^*|\mathcal{N}(S_1) \cap S_B^\mathsf{m}(G)) / P(\mathcal{N}(S_1) \cap S_B^\mathsf{m}(G)) \geq (1-\delta) L(G,G^*|S_1) / P(S_1)$. Therefore,
\begin{align*}
    L_{\mathsf{m}}(G,G^*|S_2) \geq & L_{\mathsf{m}}(G,G^*|\mathcal{N}(S_1) \cap S_B^\mathsf{m}(G)) - L_{\mathsf{m}}(G,G^*|S_1)\\
    \geq & (1-\delta)\cdot \frac{P(\mathcal{N}(S_1) \cap S_B^\mathsf{m}(G))}{P(S_1)} \cdot L_{\mathsf{m}}(G,G^*|S_1) - L_{\mathsf{m}}(G,G^*|S_1) \\
    \geq & (1-\delta)\cdot \frac{P(\mathcal{N}(S_1))- P(\overline{S_B^\mathsf{m}(G)})}{P(S_1)} \cdot L_{\mathsf{m}}(G,G^*|S_1) - L_{\mathsf{m}}(G,G^*|S_1)\\
    \geq & (1-\delta)\cdot \frac{P(\mathcal{N}(S_1))}{P(S_1)}\cdot L_{\mathsf{m}}(G,G^*|S_1) - \frac{L_{\mathsf{m}}(G,G^*|S_1)}{P(S_1)}\cdot P(\overline{S_B^\mathsf{m}(G)})  - L_{\mathsf{m}}(G,G^*|S_1) \\
    \geq & (1-\delta)c \cdot L(G,G^*|S_1) - P(\overline{S_B^\mathsf{m}(G)}) - L_{\mathsf{m}}(G,G^*|S_1) \\
    = & (c(1-\delta)-1) L_{\mathsf{m}}(G,G^*|S_1) - R_B^\mathsf{m}(G)
\end{align*}
Now, substituting $L(G,G^*|S_2)$ on the LHS with Lemma~\ref{lemma:relation-of-loss} and noticing $\mathcal{M}(G_{pl}) \backslash \overline{S_B^\mathsf{m}(G)} \subseteq S_1$, with simple transformation we have:
\begin{align*}
    (c(1-\delta)-2) L(G,G^*|S_1) \leq& L_{\mathsf{m}}(G,G_{pl}) - \text{Err}(G_{pl}) + 2*R_B^\mathsf{m}(G) \\
    \leq & \alpha
\end{align*}
The last inequality comes from the condition in the lemma. Thus, with Fact~\ref{fact:beta-ineq}, we know $L(G,G^*|S_1) \leq \alpha / ((c(1-\delta)-2)) \leq q$. Now, we can bound the overall error:
\begin{align*}
    \text{Err}_\mathsf{m}(G) =& L_{\mathsf{m}}(G,G^*|S_1) + L_{\mathsf{m}}(G,G^*|S_2) + L_{\mathsf{m}}(G,G^*|\overline{S_B^\mathsf{m}(G)}) \\
    \leq & q + (q + L_{\mathsf{m}}(G,G_{pl}) - \text{Err}(G_{pl}) + R_B^\mathsf{m}(G)) + R_B^\mathsf{m}(G) \\
    \leq & 2(q+R_B^\mathsf{m}(G)) + L_{\mathsf{m}}(G,G_{pl}) - \text{Err}(G_{pl})
\end{align*}
\end{proof}

Now, we will prove our main lemma, based on which we will be able to derive \Cref{thm:main2} and \Cref{thm:main}.
\begin{lemma}
\label{thm:main-lemma}
Suppose \Cref{thm:expansion} holds true. Then we can bound:
\begin{align*}
    \text{Err}_\mathsf{m}(G) \leq L(G) \triangleq \frac{c+3}{c-1} L_{m} (G,G_{pl})+ \frac{2c+2}{c-1} R_{\mathcal{B}}^m(G;\delta) - \text{Err}(G_{pl}).
\end{align*}
for any $\delta \in [0,\frac1c]$.
\end{lemma}
\begin{proof}
First, we consider the case where $L_{\mathsf{m}}(G,G_{pl}) + 2 R_B^\mathsf{m}(G) \leq \frac{c-1}{2}\cdot \text{Err}(G_{pl})$. In this case, we can find some $\beta \in (0,\frac{c-1}{2}]$ such that.
\begin{align*}
    L_{\mathsf{m}}(G,G_{pl})+2R_B^\mathsf{m}(G) = \beta \text{Err}(G_{pl}) = \text{Err}(G_{pl}) + (\beta-1) \text{Err}(G_{pl})
\end{align*}
Thus, by lemma~\ref{lemma:main-bound}, we have:
\begin{align*}
    \text{Err}_\mathsf{m}(G) &\leq 2(\frac{\beta}{c-1} \text{Err}(G_{pl}) +R_B^\mathsf{m}(G)) + L_{\mathsf{m}}(G,G_{pl}) - \text{Err}(G_{pl}) \\
    &= \frac{2}{c-1} \beta\text{Err}(G_{pl}) + 2R_B^\mathsf{m}(G) + L_{\mathsf{m}}(G,G_{pl}) - \text{Err}(G_{pl}) \\
    &= \frac{2}{c-1} (L_{\mathsf{m}}(G,G_{pl}) + 2R_B^\mathsf{m}(G)) + 2R_B^\mathsf{m}(G) + L_{\mathsf{m}}(G,G_{pl}) - \text{Err}(G_{pl}) \\
    &\leq \frac{c+3}{c-1} L_{\mathsf{m}}(G,G_{pl}) + \frac{2c+2}{c-1} R_B^\mathsf{m}(G) - \text{Err}(G_{pl}) \\
    &= L(G)
\end{align*}

Next, we consider the case where $L_{\mathsf{m}}(G,G_{pl}) + 2 R_B^\mathsf{m}(G) > \frac{c-1}{2}\cdot \text{Err}(G_{pl})$. By triangle inequality, we have:
\begin{align*}
    \text{Err}_\mathsf{m}(G) = L_{\mathsf{m}}(G,G^*) &\leq L_{\mathsf{m}}(G,G_{pl}) + L_{\mathsf{m}}(G_{pl},G^*) \\
    &= L_{\mathsf{m}}(G,G_{pl}) + 2 \text{Err}(G_{pl}) - \text{Err}(G_{pl}) \\
    &< L_{\mathsf{m}}(G,G_{pl}) + \frac{4}{c-1} (L_{\mathsf{m}}(G,G_{pl}) + 2 R_B^\mathsf{m}(G)) - \text{Err}(G_{pl})\\
    &= \frac{c+3}{c-1} L_{\mathsf{m}}(G,G_{pl}) + \frac{8}{c-1}R_B^\mathsf{m}(G) - \text{Err}(G_{pl}) \\
    &\leq \frac{c+3}{c-1} L_{\mathsf{m}}(G,G_{pl}) + \frac{2c+2}{c-1}R_B^\mathsf{m}(G) - \text{Err}(G_{pl}) & \text{(using $c>3$)} \\
    &= L(G)
\end{align*}
\end{proof}

Now, we provide the proof of \Cref{thm:main2} and \Cref{thm:main} based on \Cref{thm:main-lemma}.
\begin{proof} [Proof of \Cref{thm:main2}]
   Since $\hat{G}$ is an optimizer of $L(G)$, we know $\text{Err}_\mathsf{m}(\hat{G}) \leq L(\hat{G}) \leq L(G^*)$. Substituting $G^*$ into $L(G)$ gives the bound in the theorem.
\end{proof}

\begin{proof} [Proof of \Cref{thm:main}]
Note that $\text{CertR}(G) \geq \frac{0.5-Err(G)}{Lip(G)}$ by its definition. Substituting $Err(G) \leq L(G)$ gives the inequality in the theorem.
\end{proof}





\section{Error Control of Newton's Method}
\label{subsec:newton-error-control}

    Recall that given an unconstrained matrix $V\in \sR^{n\times n}$, we know that $W=(VV^\T)^{-\frac 1 2}V$ is orthogonal, i.e., $||W||_2 = 1$, which provides certified robustness for the resulting model.
    In practice, we use a finite number of Newton's iteration steps to approximate $(VV^\T)^{-\frac 1 2}$.
    In this section, we provide the following theorem which rigorously control the spectral norm under finite Newton's iteration steps.
    
    \begin{adxtheorem}
        Given a matrix $V \in \sR^{n\times n}$ such that $||I-VV^\T||_2 < 1$, if we use Newton's iteraton for $k^\star$ steps following \Cref{eq:newton-iter} with initialization $Y_0 = VV^\T$ and $Z_0 = I$, then we have
        \begin{equation}
            || Z_{k^\star} V ||_2 \le 1 + \dfrac{||V||_2}{\sqrt{\rho_{\min}(VV^\T)}} (1-\sqrt{1-||I-VV^\T||_2^{2^{k^\star}}})
            \le 1 + \dfrac{||V||_2}{\sqrt{\rho_{\min}(VV^\T)}} ||I-VV^\T||_2^{2^{k^\star}},
        \end{equation}
        where $\rho_{\min}$ is the smallest eigenvalue of the matrix.
        \label{thm:err-control}
    \end{adxtheorem}
    
    
    \begin{remark}
        We make sure the condition $||I-VV^\T||_2 < 1$ is satisfied by rescaling as discussed in \Cref{subsec:method-ot}.
        As the theorem shows, along with the increase of Newton's iteration step $k^\star$, the spectral norm of $Z_{k^\star} V$ approaches $1$ where the additional term $||I-VV^\T||_2^{2^{k}}$ decays exponentially.
        Hence, we can rigorously bound the error in the orthogonalization process caused by finite steps and use the bound to determine how many finite steps are needed.
        Indeed, in practice, we apply singular value decomposition to the computed matrix $Z_{k^\star} V$, and find its maximum singular value always approaches $1$ from the left side, i.e., the actual spectral norm is equal to or smaller than $1$.
        Detail experimental verification is in \Cref{sec:app-err-control}.
    \end{remark}

\subsection{Proof of Theorem \ref{thm:err-control}}

    \label{adxsec:err-control-prf}

For brevity, throughout this section, for $V\in \sR^{n\times n}$, we define $A = VV^\T$.
Note that $A$ is a real symmetric matrix.
Then, we recursively define 
\begin{equation}
    \begin{aligned}
        & B_0 = I, \\
        & B_{k+1} = \frac 1 2 \left( 3 B_k - B_k^3 A \right).
    \end{aligned}
    \label{eq:B}
\end{equation}

Before proving the main theorem, we first present the following three lemmas.

\begin{lemma}
    For any $k\in \sN$, $B_k A = A B_k$ and $B_k A^{\frac 1 2} = A^{\frac 1 2} B_k$.
    \label{lemma:b-1}
\end{lemma}

\begin{proof}[Proof of \Cref{lemma:b-1}]
    We prove the lemma by induction.
    Since $B_0 = I$, for $k = 0$ the lemma holds.
    Suppose that the lemma holds for $k$, then we have
    \begin{equation}
        \begin{aligned}
            & B_{k+1}A = \frac 1 2(3B_k - B_k^3A) A = \frac 1 2 (3B_kA - B_k^3 A^2) \overset{(*)}{=} \frac 1 2 (3AB_k - A B_k^3 A) = A \cdot \frac 1 2 (3B_k - B_k^3 A) = AB_{k+1}, \\
            & B_{k+1}A^{\frac 1 2} = \frac 1 2 (3B_k - B_k^3 A) A^{\frac 1 2} = \frac 1 2 (3B_kA^{\frac 1 2} - B_k^3 A^{\frac 3 2}) \overset{(*)}{=} \frac 1 2 (3A^{\frac 1 2} B_k - A^{\frac 1 2} B_k^3 A) = A^{\frac 1 2} \cdot \frac 1 2 (3B_k - B_k^3 A) = A^{\frac 1 2} B_{k+1}.
        \end{aligned}
    \end{equation}
    In above equations, $(*)$ is due to the induction assumption that $B_kA = AB_k$ or $B_kA^{\frac 1 2} = A^{\frac 1 2}B_k$.
    Therefore, the lemma holds with $k+1$ and by induction the lemma holds for any $k\in \sN$.
\end{proof}

\begin{lemma}
    For any $k \in \sN$, $Y_k = B_k A$ and $Z_k = B_k$.
    \label{lemma:b-2}
\end{lemma}

\begin{proof}[Proof of \Cref{lemma:b-2}]
    We prove the lemma by induction.
    Since $Y_0 = VV^\T = A$, $Z_0 = I$, and $B_0 = I$, for $k=0$ the lemma holds.
    Suppose that the lemma holds for $k$, then we have
    \begin{equation}
        \begin{aligned}
            & Y_{k+1} = \frac 1 2 Y_k (3I - Z_kY_k) = \frac 1 2 B_k A (3I - B_k^2 A) \overset{(*)}{=} \frac 1 2 (3B_k - B_k^3 A) \cdot A = B_{k+1} A, \\
            & Z_{k+1} = \frac 1 2 (3I - Z_kY_k) Z_k = \frac 1 2 (3I - B_k^2 A) B_k \overset{(*)}{=} \frac 1 2 (3B_k - B_k^3 A) = B_{k+1},
        \end{aligned}
    \end{equation}
    where $(*)$ leverages $B_kA = AB_k$ from \Cref{lemma:b-1}.
    Therefore, by induction the lemma holds for any $k\in \sN$.
\end{proof}

\begin{lemma}
    When $||I - A||_2 < 1$, for any $k\in \sN$,
    the eigenvalue $\lambda\in \sR$ of matrix $A^{\frac 1 2} B_k$ is positive.
    \label{lemma:b-3}
\end{lemma}

\begin{proof}[Proof of \Cref{lemma:b-3}]
    We define $C_k = A^{\frac 1 2} B_k$, then by leveraging the commutability between $B_k$ and $A$/$A^{\frac 1 2}$~(\Cref{lemma:b-1}) we have the following iteration:
    \begin{equation}
        \left\{
        \begin{aligned}
            & C_0 = A^{\frac 1 2}, \\
            & C_{k+1} = \frac 1 2 (3C_k - C_k^3).
        \end{aligned}
        \right.
    \end{equation}
    Since $||I-A||_2 < 1$ and $A$ is a real symmetric matrix, any eigenvalue of $C_0$, denoted by $\lambda_i^{C_0}$, $\in (0,\sqrt 2)$.
    Denoting the diagonalization of $C_0$ by $C_0 = P^{-1} D_0 P$ where $D_0 = \diag(\lambda_1^{C_0}, \cdots, \lambda_n^{C_0})$.
    Then, from the iteration, we have $C_{k+1} = P^{-1} \left(\frac 1 2(3D_k - D_k^3)\right) P$ and therefore $\lambda_i^{C_{k+1}} = \lambda_i^{C_k} \left( \frac 3 2 - \frac 1 2 (\lambda_i^{C_k})^2 \right)$.
    Define function $f: \sR \to \sR$ such that $f(x) = x(\frac 3 2 - \frac 1 2 x^2)$. 
    We find $f'(x) = 0 \Rightarrow x \in \{0,1\}$.
    Thus, when $x \in (0,\sqrt{2})$, $f(x) \in (0,1) \subseteq (0,\sqrt{2})$.
    Now we apply the induction.
    When $k=0$, we have $\lambda_i^{C_k} \in (0,\sqrt{2})$.
    Suppose $\lambda_i^{C_k} \in (0, \sqrt{2})$, from above result we have $\lambda_i^{C_{k+1}} = f(\lambda_i^{C_k}) \in (0,\sqrt{2})$.
    Thus, for any $k\in \sN$, all eigenvalues of $C_k = A^{\frac 1 2} B_k$ are positive.
\end{proof}

Now we are ready to prove the main theorem.

\begin{proof}[Proof of Theorem \ref{thm:err-control}]
   Since $Z_k = B_k$ for any $k\in \sN$ acccording to \Cref{lemma:b-2}, we focus on $B_k$ and its expression of iteration~(\Cref{eq:B}) henceforth.
   According to \cite[Section 6]{bini2005algorithms}, define $R_k = I - B_k^2 A$, we have
    $R_{k+1} = \frac{3}{4}R_k^2 + \frac{1}{4}R_k^3$.
    We have $||I-A||_2 = ||R_0||_2 < 1$.
    By induction,  
    \begin{equation}
        \begin{aligned}
            ||R_{k+1}||_2 & \le \frac{3}{4} ||R_{k}||_2^2 + \frac{1}{4}||R_{k}||_2^3 \\
            & \le \frac{3}{4} ||R_{k}||_2^2 + \frac{1}{4}||R_{k}||_2^2 & \text{(by induction condition $||R_k||_2 \le 1$)} \\
            & = ||R_k||_2^2.
        \end{aligned}
    \end{equation}
    Therefore, $||R_{k^\star}||_2 \le ||I-A||_2^{2^{k^\star}}$, i.e., the eigenvalues of the symmetric matrix $I - B_{k^\star}^2A$ are in the range $[-||I-A||_2^{2^{k^\star}}, ||I-A||_2^{2^{k^\star}}]$.
    
    Given an eigenvalue $\lambda \in \sR$ with eigenvector $x\in\sR^n$ of the real symmetric matrix $A^{\frac 1 2}B_{k^\star}$,
    we have
    \begin{equation}
        \begin{aligned}
            & A^{\frac 1 2}B_{k^\star}x = \lambda x \\
            \iff & B_{k^\star}^2 A x = \lambda^2 x & \text{(by \Cref{lemma:b-1})} \\
            \iff & (I-B_{k^\star}^2A)x = \left(1-\lambda^2\right) x.
        \end{aligned}
    \end{equation}
    Thus, $(1-\lambda^2)$ is an eigenvalue of matrix $(I-B_{k^\star}^2A)$.
    Since $(1-\lambda^2) \in [-||I-A||_2^{2^{k^\star}}, ||I-A||_2^{2^{k^\star}}]$, we get 
    \begin{equation}
        \min\{|\lambda-1|, |\lambda+1|\} \le 1 - \sqrt{1 - ||I-A||_2^{2^{k^\star}}}.
    \end{equation}
    
    Then, according to \Cref{lemma:b-3}, we know that $\lambda > 0$ and hence $\lambda \in \left[\sqrt{1 - ||I-A||_2^{2^{k^\star}}}, 2-\sqrt{1 - ||I-A||_2^{2^{k^\star}}}\right]$.
    Now we apply diagonalization to $A^{\frac 1 2}B_{k^\star}$:
    \begin{equation}
        A^{\frac 1 2} B_{k^\star} := P^\T \Lambda P.
    \end{equation}
    As a result,
    \begin{equation}
        \begin{aligned}
            & ||B_{k^\star} V||_2 \\
            = & ||A^{-\frac 1 2} P^\T \Lambda P V ||_2 \\
            = & || A^{-\frac 1 2} P^\T (\Lambda - I) P V + A^{-\frac 1 2} P^\T I P V ||_2 \\
            \le & ||A^{-\frac 1 2}||_2 \cdot (1-\sqrt{1-||I-A||_2^{2^{k^\star}}}) \cdot ||V||_2 + || A^{-\frac 1 2} V ||_2 \\
            = & ||A^{-\frac 1 2}||_2 \cdot ||V||_2 \cdot (1-\sqrt{1-||I-A||_2^{2^{k^\star}}})  + 1 \\
        \end{aligned}
        \label{eq:pf-err-control-prim}
    \end{equation}
    where the last equality uses the fact $A^{-\frac 1 2} V$ is a orthogonal matrix with spectral norm $1$.
    
    Since any eigenvector with eigenvalue $\lambda$ of $A^{-\frac 1 2}$ corresponds to the eigenvector with eigenvalue $(1/\lambda^2)$ of $A = VV^\T$, and $VV^T$ as a symmetric real matrix only has real eigenvalues,
    \begin{equation}
        ||A^{-\frac 1 2}||_2 = \dfrac{1}{\sqrt{\rho_{\min}(VV^\T)}}.
    \end{equation}
    Plug it into \Cref{eq:pf-err-control-prim}, we get
    \begin{equation}
        ||Z_{k^\star} V||_2 = ||B_{k^\star} V||_2 = 1 + \dfrac{||V||_2}{\sqrt{\rho_{\min}(VV^\T)}} (1-\sqrt{1-||I-A||_2^{2^{k^\star}}}).
    \end{equation}
    Notice that $1 - \sqrt{1-x} \le x$ for $x = ||I-A||_2^{2^{k^\star}} \in [0,1]$, we conclude the proof.

\end{proof}

\section{Addition Exp Results}
\subsection{Visualization}
\label{sec:app-visualize}
We show the representation visualization on 16 neurons for SOC and \name in Figure~\ref{fig:visualize-full}.
\begin{figure}[t]
    \centering
    \includegraphics[width=0.4\textwidth]{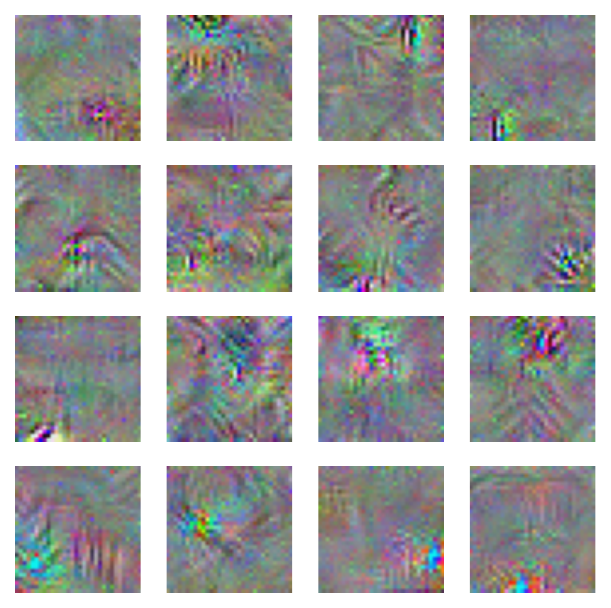}
    \hspace{0.2in}
    \includegraphics[width=0.4\textwidth]{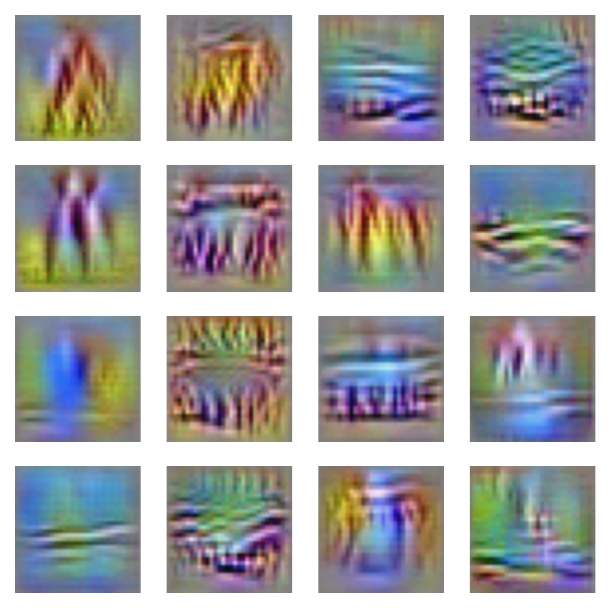}
    \caption{Visualizing the features in the last hidden layer of LipConvnet-20 for SOC (left) and \name (right). Each image corresponds to one randomly chosen neuron from the last hidden layer and is optimized to maximize the value of the neuron.}
    \label{fig:visualize-full}
\end{figure}

\subsection{Residual Connection}
\label{sec:app-residual}
We show the loss landscape for both \name and SOC in Figure~\ref{fig:contour-full}
\begin{figure}
    \centering
    \captionsetup[subfigure]{justification=centering}
    \begin{subfigure}[b]{0.24\textwidth}
         \centering
         \includegraphics[width=\textwidth]{figs/contour-cifar10_4_ot_32_maxmin_cr0.0_b256.pdf}
         \caption{\name, without \\residual connection.}
     \end{subfigure}
    \begin{subfigure}[b]{0.24\textwidth}
         \centering
         \includegraphics[width=\textwidth]{figs/contour-cifar10_4_ot_32_maxmin_cr0.0_res_b256.pdf}
         \caption{\name, with \\residual connection.}
     \end{subfigure}
    \begin{subfigure}[b]{0.24\textwidth}
         \centering
         \includegraphics[width=\textwidth]{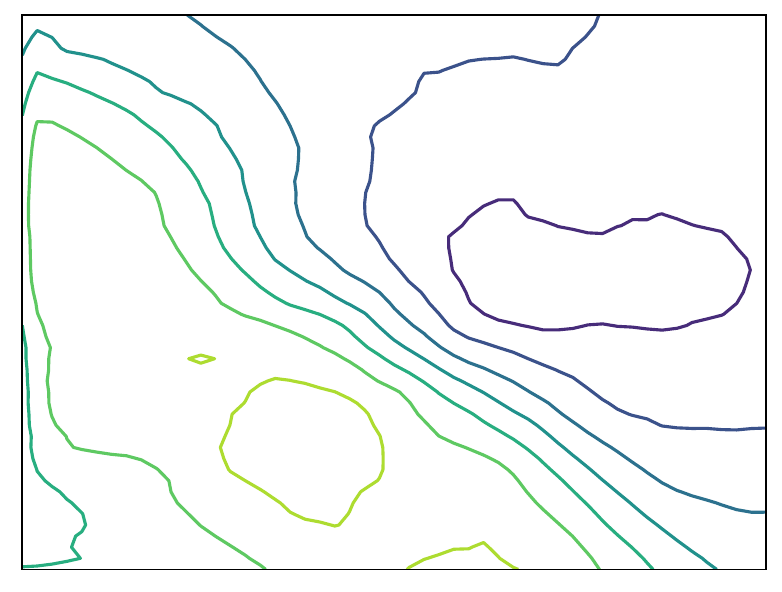}
         \caption{SOC, without \\residual connection.}
     \end{subfigure}
    \begin{subfigure}[b]{0.24\textwidth}
         \centering
         \includegraphics[width=\textwidth]{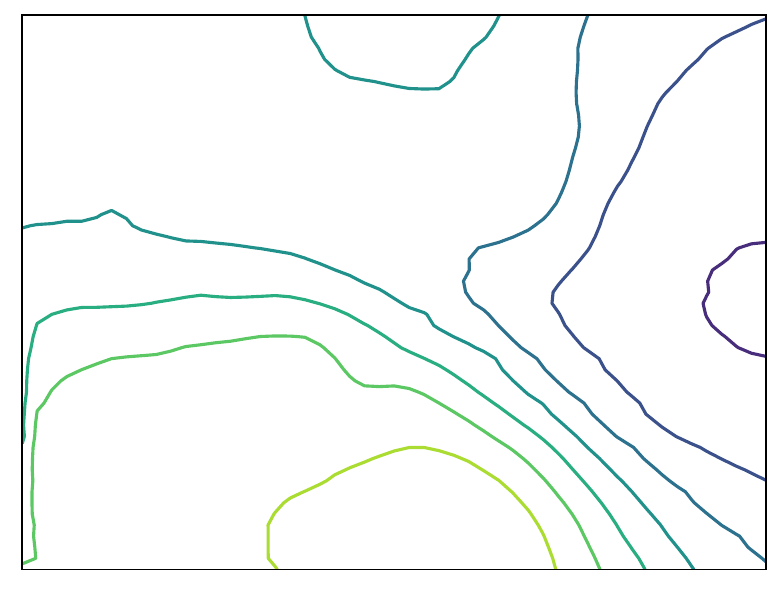}
         \caption{SOC, with \\residual connection.}
     \end{subfigure}
    \caption{The loss landscape~\cite{li2018visualizing} with respect to the parameters of \name and SOC network with and without residual connections. The figure is plotted by calculating the loss contour on two randomly chosen directions of parameters. We can observe that the residual connection greatly smoothifies the \name network, while both SOCs with and without residual connection are smooth.}
    \label{fig:contour-full}
\end{figure}

\subsection{Circular vs. Zero padding}
\label{sec:app-circ-pad}
\begin{table}[t]
    \centering
    \caption{Performance comparison of \name network with default circular padding and pre-processed zero padding.}
    \label{tab:abl-circular}
    \begin{tabular}{l l l | c | c c c}
        \toprule
        \multirow{2}{*}{\bf Model} & \bf Conv. & \multirow{2}{*}{\bf Padding} & \bf Vanilla & \multicolumn{3}{c}{\bf Certified Accuracy at $\rho=$} \\
        & \bf Type & \bf & \bf Accuracy & 36/255 & 72/255 & 108/255 \\
        \midrule
        \multirow{2}{*}{LipConvnet-20} & \multirow{2}{*}{\ot} & Circular-pad & 76.65\% & 61.15\% & 43.83\% & 28.78\% \\
        & & Zero-pad & \bf 77.86\% & \bf 63.54\% & \bf 47.15\% & \bf 32.12\% \\
		\bottomrule
    \end{tabular}
\end{table}

As we discussed in Section~\ref{subsec:method-ot}, the default parametrization leads to the convolution result with circular-padding, while we will pre-process the input to get the final result with zero-padding. In Table~\ref{tab:abl-circular}, we show the performance comparison between circular and zero padding. We can see that the performance significantly improves when we do the pre-processing and use the zero-padding instead of default circular-padding.

\subsection{Error Control of Newton's Iterations}
\label{sec:app-err-control}
As we discuss in Section~\ref{subsec:newton-error-control}, we use a finite number of Newton's iteration to approximate the inverse square root of a matrix $Z_{k^*} \approx (V V^\intercal)^{-\frac12}$. Theoretically, we have shown that the error of approximated orthogonal matrix $Z_{k^*} V$ will decay exponentially with number of iterations. Furthermore, we observe in practice that during the Newton's iteration, the maximum singular value of $Z_k V$ will always approach 1 from the left side. As an example, we show the maximum singular value ($\sigma_{max}$) for all \name layers in LipConvnet-20 at different steps ($k$) during the Newton's Iteration in Figure~\ref{fig:err-bound}. We can see that $\sigma_{max}<1$ for all the layers, so we can safely assume that the Lipschitz bound is no larger than 1. In addition, $\sigma_{max}>1-10^{-4}=0.9999$ after $k=8$ iterations for all the layers, which indicates that the Newton's iteration converges well.

\begin{figure}
    \centering
    \includegraphics[width=0.7\textwidth]{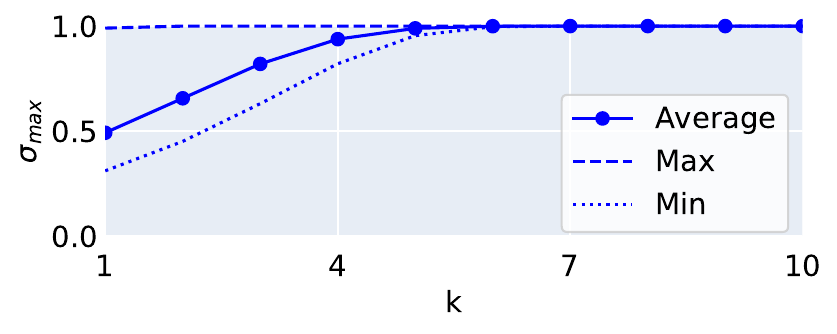}
    \caption{The maximum singular value $\sigma_{max}$ of $Z_k V$ for each of the \name layer in LipConvnet-20 during the Newton's iteration. We verify that all the $\sigma_{max}$'s are smaller than 1, and after the 8-th iteration, all the $\sigma_{max}$'s are larger than 0.9999.}
    \label{fig:err-bound}
\end{figure}

\subsection{Supervised Learning without CReg Loss and HH Activation}
\label{sec:app-cifar100}

\begin{table}[t]
    \centering
    \caption{Certified accuracy of 1-Lipschitz model without CReg loss and HH activation on CIFAR-10 in supervised setting. }
    \label{tab:sup-cifar10}
    \begin{tabular}{l l| c | c c c | c}
        \toprule
        \multirow{2}{*}{\bf Model} & \bf Conv. & \bf Vanilla & \multicolumn{3}{c|}{\bf Certified Accuracy at $\rho=$} & \bf Evaluation \\
        & \bf Type & \bf Accuracy & 36/255 & 72/255 & 108/255 & \bf Time (sec)\\
        \midrule
        \multirow{2}{*}{LipConvnet-5} & SOC & 75.78\% & 59.18\% & 42.01\% & 27.09\% & 2.117 \\
        & \ot & \bf 77.20\% & \bf 61.76\% & \bf 44.45\% & \bf 29.61\% & \bf 1.406 \\
        \midrule
        \multirow{2}{*}{LipConvnet-10} & SOC & 76.45\% & 60.86\% & 44.15\% & 29.15\% & 3.170 \\
        & \ot & \bf 77.30\% & \bf 62.54\% & \bf 46.03\% & \bf 30.64\% & \bf 1.420 \\
        \midrule
        \multirow{2}{*}{LipConvnet-15} & SOC & 76.68\% & 61.36\% & 44.28\% & 29.66\% & 3.993 \\
        & \ot & \bf 77.34\% & \bf 63.40\% & \bf 46.54\% & \bf 31.75\% & \bf 1.453 \\
        \midrule
        \multirow{2}{*}{LipConvnet-20} & SOC & 76.90\% & 61.87\% & 45.79\% & 31.08\% & 4.752 \\
        & \ot & \bf 77.86\% & \bf 63.54\% & \bf 47.15\% & \bf 32.12\% & \bf 1.558 \\
        \midrule
        \multirow{2}{*}{LipConvnet-25} & SOC & 75.24\% & 60.17\% & 43.55\% & 28.60\% & 5.613 \\
        & \ot & \bf 77.76\% & \bf 62.77\% & \bf 46.06\% & \bf 31.20\% & \bf 1.834 \\
        \midrule
        \multirow{2}{*}{LipConvnet-30} & SOC & 74.51\% & 59.06\% & 42.46\% & 28.05\% & 6.438 \\
        & \ot & \bf 77.34\% & \bf 62.76\% & \bf 46.24\% & \bf 31.07\% & \bf 2.219 \\
        \midrule
        \multirow{2}{*}{LipConvnet-35} & SOC & 73.73\% & 58.50\% & 41.75\% & 27.20\% & 7.400 \\
        & \ot & \bf 77.54\% & \bf 62.62\% & \bf 46.28\% & \bf 31.64\% & \bf 2.620 \\
        \midrule
        \multirow{2}{*}{LipConvnet-40} & SOC & 71.63\% & 54.39\% & 37.92\% & 24.13\% & 8.175 \\
        & \ot & \bf 77.79\% & \bf 62.69\% & \bf 46.34\% & \bf 31.32\% & \bf 2.910 \\
		\bottomrule
    \end{tabular}
\end{table}

\begin{table}[t]
    \centering
    \caption{Certified accuracy of 1-Lipschitz model without CReg loss, HH activation and LLN on CIFAR-100 in supervised setting. }
    \label{tab:sup-cifar100}
    \begin{tabular}{l l| c | c c c}
        \toprule
        \multirow{2}{*}{\bf Model} & \bf Conv. & \bf Vanilla & \multicolumn{3}{c}{\bf Certified Accuracy at $\rho=$} \\
        & \bf Type & \bf Accuracy & 36/255 & 72/255 & 108/255 \\
        \midrule
        \multirow{2}{*}{LipConvnet-5} & SOC & 42.71\% & 27.86\% & 17.45\% & 9.99\% \\
        & \ot & \bf 46.07\% & \bf 31.28\% & \bf 19.86\% & \bf 12.17\% \\
        \midrule
        \multirow{2}{*}{LipConvnet-10} & SOC & 43.72\% & 29.39\% & 18.56\% & 11.16\% \\
        & \ot & \bf 44.68\% & \bf 30.59\% & \bf 19.69\% & \bf 12.33\% \\
        \midrule
        \multirow{2}{*}{LipConvnet-15} & SOC & 42.92\% & 28.81\% & 17.93\% & 10.73\%  \\
        & \ot & \bf 46.01\% & \bf 32.08\% & \bf 20.72\% & \bf 12.92\% \\
        \midrule
        \multirow{2}{*}{LipConvnet-20} & SOC & 43.06\% & 29.34\% & 18.66\% & 11.20\%  \\
        & \ot & \bf 46.05\% & \bf 32.17\% & \bf 20.81\% & \bf 13.16\% \\
        \midrule
        \multirow{2}{*}{LipConvnet-25} & SOC & 43.37\% & 28.59\% & 18.18\% & 10.85\% \\
        & \ot & \bf 46.21\% & \bf 31.81\% & \bf 21.01\% & \bf 12.83\% \\
        \midrule
        \multirow{2}{*}{LipConvnet-30} & SOC & 42.87\% & 28.74\% & 18.47\% & 11.21\%  \\
        & \ot & \bf 45.71\% & \bf 32.23\% & \bf 20.87\% & \bf 13.03\%  \\
        \midrule
        \multirow{2}{*}{LipConvnet-35} & SOC & 42.42\% & 28.34\% & 18.10\% & 10.96\%  \\
        & \ot & \bf 45.38\% & \bf 31.03\% & \bf 20.02\% & \bf 12.46\%  \\
        \midrule
        \multirow{2}{*}{LipConvnet-40} & SOC & 41.84\% & 28.00\% & 17.40\% & 10.28\%  \\
        & \ot & \bf 45.30\% & \bf 30.91\% & \bf 19.97\% & \bf 12.36\%  \\
		\bottomrule
    \end{tabular}
\end{table}

We show the results of semi-supervised learning under standard setting (without CReg Loss, HH Activation, and LLN on CIFAR-100) in \Cref{tab:sup-cifar10} and \Cref{tab:sup-cifar100}. We can observe that we still achieve a good performance compared with SOC, and the gap is sometimes larger than with the different optimization techniques.

\subsection{Full semi-supervised}
\label{sec:app-semi-sup}

\begin{table}[h]
    \centering
    \caption{Certified accuracy of 1-Lipschitz model with CReg loss and HH activation on CIFAR-10 in semi-supervised setting.}
    \label{tab:ssl-full-cifar10-crhh}
    \begin{tabular}{l l| c | c c c}
        \toprule
        \multirow{2}{*}{\bf Model} & \bf Conv. & \bf Vanilla & \multicolumn{3}{c}{\bf Certified Accuracy at $\rho=$} \\
        & \bf Type & \bf Accuracy & 36/255 & 72/255 & 108/255 \\
        \midrule
        \multirow{2}{*}{LipConvnet-5 + CReg + HH} & SOC & 69.67\% & 59.28\% & 48.02\% & 38.30\% \\
        & \ot & \bf 71.52\% & \bf 61.25\% & \bf 50.66\% & \bf 40.31\%  \\
        \midrule
        \multirow{2}{*}{LipConvnet-10 + CReg + HH} & SOC & 71.10\% & 60.81\% & 50.61\% & 41.03\% \\
        & \ot & \bf 71.82\% & \bf 62.60\% & \bf 51.76\% & \bf 42.31\%   \\
        \midrule
        \multirow{2}{*}{LipConvnet-15 + CReg + HH} & SOC & 71.15\% & 61.65\% & 51.78\% & 42.53\% \\
        & \ot & \bf 71.89\% & \bf 62.37\% & \bf 52.24\% & \bf 42.71\%  \\
        \midrule
        \multirow{2}{*}{LipConvnet-20 + CReg + HH} & SOC & 70.95\% & 61.72\% & 51.78\% & 42.01\% \\
        & \ot & \bf 71.86\% & \bf 62.86\% & \bf 52.24\% & \bf 42.39\% \\
        \midrule
        \multirow{2}{*}{LipConvnet-25 + CReg + HH} & SOC & 70.58\% & 61.40\% & 51.46\% & 42.05\% \\
        & \ot & \bf 71.64\% & \bf 62.67\% & \bf 52.00\% & \bf 42.34\% \\
        \midrule
        \multirow{2}{*}{LipConvnet-30 + CReg + HH} & SOC & 70.37\% & 60.74\% & 51.23\% & 41.81\% \\
        & \ot & \bf 72.08\% & \bf 62.84\% & \bf 51.99\% & \bf 41.94\% \\
		\bottomrule
    \end{tabular}
\end{table}

\begin{table}[h]
    \centering
    \caption{Certified accuracy of 1-Lipschitz model without CReg loss and HH activation on CIFAR-10 in semi-supervised setting. }
    \label{tab:ssl-full-cifar10}
    \begin{tabular}{l l| c | c c c}
        \toprule
        \multirow{2}{*}{\bf Model} & \bf Conv. & \bf Vanilla & \multicolumn{3}{c}{\bf Certified Accuracy at $\rho=$} \\
        & \bf Type & \bf Accuracy & 36/255 & 72/255 & 108/255 \\
        \midrule
        \multirow{2}{*}{LipConvnet-5} & SOC & 70.67\% & 59.14\% & 45.88\% & 34.21\% \\
        & \ot & \bf 72.86\% & \bf 61.66\% & \bf 48.84\% & \bf 36.67\% \\
        \midrule
        \multirow{2}{*}{LipConvnet-10} & SOC & 71.86\% & 60.15\% & 48.01\% & 35.66\% \\
        & \ot & \bf 73.57\% & \bf 62.62\% & \bf 50.26\% & \bf 37.82\% \\
        \midrule
        \multirow{2}{*}{LipConvnet-15} & SOC & 73.05\% & 62.11\% & 50.44\% & 38.50\% \\
        & \ot & \bf 73.74\% & \bf 63.01\% & \bf 50.66\% & \bf 38.73\% \\
        \midrule
        \multirow{2}{*}{LipConvnet-20} & SOC & 72.77\% & 62.51\% & 50.40\% & 38.05\% \\
        & \ot & \bf 73.64\% & \bf 63.37\% & \bf 51.00\% & \bf 38.52\% \\
        \midrule
        \multirow{2}{*}{LipConvnet-25} & SOC & 72.45\% & 62.03\% & 50.12\% & 38.28\% \\
        & \ot & \bf 73.62\% & \bf 63.61\% & \bf 51.21\% & \bf 38.77\% \\
        \midrule
        \multirow{2}{*}{LipConvnet-30} & SOC & 71.04\% & 61.06\% & 49.30\% & 38.11\% \\
        & \ot & \bf 71.71\% & \bf 63.32\% & \bf 51.27\% & \bf 39.42\% \\
		\bottomrule
    \end{tabular}
\end{table}

We show the results for all architectures with semi-supervised learning in Table~\ref{tab:ssl-full-cifar10-crhh} (with CReg and HH) and \ref{tab:ssl-full-cifar10} (without CReg and HH). We can see observe that we still achieve a better performance compared with SOC. In addition, these results also have a larger certified accuracy at larger radius compared with the supervised learning setting.



{

\section{Results on TinyImageNet}

\begin{table}[t]
    \centering
    \caption{Certified accuracy of 1-Lipschitz model without CReg loss and HH activation on TinyImageNet in supervised setting. }
    \label{tab:sup-tiny}
    \begin{tabular}{l l| c | c c c }
        \toprule
        \multirow{2}{*}{\bf Model} & \bf Conv. & \bf Vanilla & \multicolumn{3}{c}{\bf Certified Accuracy at $\rho=$} \\
        & \bf Type & \bf Accuracy & 36/255 & 72/255 & 108/255 \\
        \midrule
        \multirow{2}{*}{LipConvnet-5} & SOC & 30.77\% & 19.74\% & 11.60\% & 6.89\% \\
        & \ot & \bf 32.71\% & \bf 21.44\% & \bf 12.96\% & \bf 7.92\%  \\
        \midrule
        \multirow{2}{*}{LipConvnet-10} & SOC & 31.94\% & 21.21\% & 12.80\% & \bf 7.79\% \\
        & \ot & \bf 32.31\% & \bf 21.22\% & \bf 12.96\% & 7.75\% \\
        \midrule
        \multirow{2}{*}{LipConvnet-15} & SOC & 32.26\% & 21.36\% & 12.94\% & 7.80\% \\
        & \ot & \bf 33.14\% & \bf 22.21\% & \bf 13.34\% & \bf 8.12\%  \\
        \midrule
        \multirow{2}{*}{LipConvnet-20} & SOC & 32.44\% & 21.27\% & 12.90\% & 7.63\%  \\
        & \ot & \bf 33.19\% & \bf 22.02\% & \bf 13.42\% & \bf 8.12\% \\
        \midrule
		\bottomrule
    \end{tabular}
\end{table}
To further validate the performance of \name, we evaluate it on the TinyImageNet dataset which consists of 100,000 $64\times 64$ images in 200 classes. We use the same model architecture and training setting as before and show the comparison between \name and SOC in Table \ref{tab:sup-tiny}. We can observe that LOT still outperforms the existing SOC approach in most cases. Meanwhile, we observe that all 1-Lipschitz models have a performance drop on the larger dataset compared with vanilla models, and we leave the breakthrough as future work.

\section{Discussion of Model Architectures}



\begin{table}[t]
    \centering
     \caption{\small Certified accuracy of 1-Lipschitz ResNet-18 model without CReg loss and HH activation on CIFAR-10/TinyImageNet in supervised setting}
     \resizebox{\textwidth}{!}{
    \begin{tabular}{l l|| c | c c c || c | c c c}
        \toprule
        \multirow{3}{*}{\bf Model} & \multirow{3}{*}{\bf\shortstack{Conv.\\ Type}} & \multicolumn{4}{c||}{CIFAR-10} & \multicolumn{4}{c}{TinyImageNet} \\
        \cline{3-10}
        & & \bf Vanilla & \multicolumn{3}{c||}{\bf Certified Accuracy at $\rho=$} & \bf Vanilla & \multicolumn{3}{c}{\bf Certified Accuracy at $\rho=$} \\
        & & \bf Accuracy & 36/255 & 72/255 & 108/255 & \bf Accuracy & 36/255 & 72/255 & 108/255 \\
        \midrule
        \multirow{2}{*}{\shortstack{ResNet-18}} & SOC & 66.43\% & 43.00\% & 23.00\% & 9.52\% & 23.26\% & 11.64\% & 5.21\% & 2.17\% \\
        & \ot &\bf  \bf 68.85\% & \bf 45.46\% & \bf 25.43\% & \bf 11.35\% & \bf 25.09\% & \bf 12.83\% & \bf 5.90\% & \bf 2.62\% \\
        \hline
    \end{tabular}
    }
    \label{tab:sup-resnet}
\end{table}
We would like to emphasize that there are two major differences in designing 1-Lipschitz networks and standard networks. First, standard architectures focus a lot on regularizing the network so that it does not overfit (e.g. with dropout), while in 1-Lipschitz networks we do not need strong regularization because the 1-Lipschitz constraint is already strong enough. Second, standard architectures are carefully designed so that the gradient will propagate in a proper manner (e.g. residual connection, BatchNorm), while 1-Lipschitz networks have an intrinsic gradient-norm-preserving property (see \cite{li2019preventing}). Thus, it may not be appropriate to directly utilize modern architectures in the design of 1-Lipschitz networks. In Table~\ref{tab:sup-resnet}, we show the results of a 1-Lipschitz ResNet18 model, where we (1) replace all the conv layers with 1-Lipschitz convolution, (2) replace all the residual connection with average and (3) enforce the variance parameter of BatchNorm layers to be 1 so that the model is 1-Lipschitz. We can observe that \name still outperforms SOC on both CIFAR-10 and TinyImageNet, while the performance of the ResNet18 architecture is not as good as that of LipConvNet.

}
\end{document}